\documentclass[journal]{IEEEtran}
\usepackage{generic}
\usepackage{cite}
\usepackage{amsmath,amssymb,amsfonts}
\usepackage{algorithmic}
\usepackage{graphicx}
\usepackage{algorithm,algorithmic}
\usepackage{hyperref}
\hypersetup{hidelinks=true}
\usepackage{textcomp}
\usepackage{multirow}
\usepackage{eso-pic}

\newtheorem{proposition}{Proposition}
\newtheorem{remark}{Remark}
\newtheorem{example}{Example}
\newenvironment{proof}{\noindent\textit{Proof.}}{\hfill$\square$}

\def\BibTeX{{\rm B\kern-.05em{\sc i\kern-.025em b}\kern-.08em
    T\kern-.1667em\lower.7ex\hbox{E}\kern-.125emX}}
\begin{document}
\title{Constructing VAE Latent Spaces with Prescribed Topology}
\author{Jilles S. van Hulst,
        Jakub M. Tomczak,
        W.P.M.H. (Maurice) Heemels,
        and Duarte J. Antunes
\thanks{J. S. van Hulst, W.P.M.H. Heemels, and D. J. Antunes are with the Control Systems Technology Section, Department of Mechanical Engineering, Eindhoven University of Technology, 5600 MB Eindhoven, The Netherlands. Corresponding author: {\tt j.s.v.hulst@tue.nl}.}%
\thanks{J. M. Tomczak is with the Nature Innovation Laboratory (NatInLab).}%
\thanks{The research is carried out as part of the ITEA4 20216 ASIMOV project. The ASIMOV activities are supported by the Netherlands Organisation for Applied Scientific Research TNO and the Dutch Ministry of Economic Affairs and Climate (project number: AI211006). The research leading to these results is partially funded by the German Federal Ministry of Education and Research (BMBF) within the project ASIMOV-D under grant agreement No. 01IS21022G [DLR], based on a decision of the German Bundestag.}%
}

\AddToShipoutPictureBG*{%
    \AtPageUpperLeft{%
        \setlength\unitlength{1in}%
        \hspace*{\dimexpr0.5\paperwidth\relax}
        \makebox(0,-0.8)[c]{%
            \begin{tabular}{c}
                J.S. van Hulst \emph{et al.}, ``Constructing VAE Latent Spaces with Prescribed Topology.''
                \copyright~The Authors.
            \end{tabular}%
        }%
    }%
}%

\maketitle

\begin{abstract}
Variational autoencoders (VAEs) learn low-dimensional latent representations of high-dimensional data. When the data lies on a manifold with non-Euclidean topology, the standard Gaussian prior introduces a topological mismatch that degrades reconstruction quality and prevents faithful representation. We present a constructive mathematical framework that resolves this mismatch for all manifolds that admit a product covering space. These are manifolds expressible as products of elementary factors (circles, intervals, or lines) or as quotients of such products by a finite symmetry group. The class includes cylinders, tori, M\"{o}bius strips, Klein bottles, and real projective spaces. Factorized distributions over the elementary factors yield product topologies with closed-form, decoupled KL divergences, so that each latent factor can be shaped independently while keeping training tractable. We catalogue reparametrizable encoder-prior pairs for periodic, bounded, and unbounded supports, and provide coordinate transformations that allow standard neural networks to output non-Euclidean parameters with smooth gradients. For quotient manifolds, the decoder receives group-invariant features of the covering-space coordinates, so that identified points produce identical outputs. Anchor constraints fix the coordinate system relative to the data or create soft topological holes. Experiments on synthetic manifolds and real-image datasets (rotated and cyclically shifted MNIST) confirm that a topology-matched prior aligns KL regularization with the data manifold. The resulting topology-aware models outperform the Gaussian baseline at all practically relevant regularization strengths. The code is available at \url{https://github.com/JvHulst/VAE-Topology}.
\end{abstract}


\section{Introduction}
\label{sec:introduction}
\IEEEPARstart{V}{ariational} autoencoders (VAEs), introduced by Kingma and Welling~\cite{kingma2014}, have become a fundamental tool for learning low-dimensional representations of high-dimensional data. The main innovation of VAEs is the \textit{reparametrization trick}~\cite{rezende2014}, which expresses latent samples as deterministic transformations of encoder outputs and independent noise, enabling gradient-based optimization of the loss criterion.

A key motivation for dimensionality reduction is that many real-world datasets, despite their high-dimensional representation, lie on or near a low-dimensional manifold. In such cases, we would like the learned latent space to faithfully represent the structure of this underlying manifold. However, the standard VAE framework assumes a Gaussian prior, which induces a Euclidean topology on the latent space. This creates a fundamental mismatch when the data manifold has non-Euclidean structure, for instance, when data lies on a sphere, torus, or bounded domain.

This mismatch has motivated several approaches to treat non-Euclidean latent spaces. Davidson et al.~\cite{davidson2018} introduced the hyperspherical VAE, which uses the von Mises-Fisher distribution for latent spaces on hyperspheres $S^{n-1}$. A follow-up~\cite{davidson2019} scaled this to higher dimensions using a product of vMF distributions. Falorsi et al.~\cite{falorsi2018} extended this to general Lie groups such as $SO(3)$, providing a principled approach for rotational latent spaces. Mathieu et al.~\cite{mathieu2019} proposed hyperbolic VAEs, where the Poincaré disc latent space is well-suited to data with hierarchical structure. Kalatzis et al.~\cite{kalatzis2020} extended this direction by learning Riemannian metric tensors to capture the geometry of learned representations. Perez Rey et al.~\cite{PerezRey2020} and Nagano et al.~\cite{Nagano2019} developed diffusion-based and wrapped normal distributions for VAEs on tori, spheres, and hyperbolic spaces. Normalizing flows on manifolds~\cite{Huh2023,Kuzina2022,Brehmer2020} take a different approach: rather than prescribing the prior topology, they compose diffeomorphisms to learn a flexible distribution, at the cost of Jacobian determinants and Monte Carlo KL estimation. Equivariant representations~\cite{Dupont2020} impose structure through architectural constraints, while topological autoencoders~\cite{Moor2021} use persistent homology to preserve topological structure in learned representations. Tomczak and Welling~\cite{tomczak2018} showed more generally that the choice of prior distribution has a key effect on the quality of learned representations, even without changing the latent topology. Miao et al.~\cite{miao2022} proposed an alternative mechanism: rather than modifying the prior, they keep a Gaussian prior and impose inductive biases through a learned parametric mapping from an intermediary latent space. Although these works have substantially advanced non-Euclidean representation learning, none provides a constructive design procedure for VAEs with a prescribed arbitrary topology while retaining a tractable closed-form ELBO. In this paper, we develop a unified framework that covers a broad class of manifolds, including many that arise in practice, and provides a systematic pipeline for constructing the corresponding VAEs.

The specific class of manifolds that we consider are those that admit a \emph{product covering space}: manifolds that can be expressed as a product of elementary factors (circles $S^1$, bounded intervals $[0,1]$, half-lines $[0,\infty)$, real lines $\mathbb{R}$, or hyperspheres $S^{n-1}$), or as a quotient of such a product by a finite symmetry group $G$. This class includes cylinders, tori, annuli, and hypercubes as product manifolds, and M\"{o}bius strips, Klein bottles, and real projective spaces as finite quotients of products. Product topologies arise naturally when data has periodic or bounded factors of variation. For example, dihedral angles in protein backbone geometry live on tori~\cite{Ramachandran1963,Boomsma2008}, and joint angles in robotics produce similar toroidal configuration spaces~\cite{LaValle2006}. Quotient topologies arise when a discrete symmetry identifies points in the product space. For instance, non-orientable topologies such as M\"{o}bius strips and Klein bottles arise in the analysis of natural image patches~\cite{Carlsson2008}, and real projective spaces describe the orientations of objects with head-to-tail symmetry~\cite{Gilitschenski2020}.

Our main contribution is a constructive framework that provides topology-aware VAEs with tractable ELBO objectives and end-to-end gradient-based training for any manifold in this class. Based on this result, we provide a systematic pipeline for building the corresponding VAE. The pipeline consists of the following steps:
\begin{enumerate}
    \item Express $\mathcal{M}$ as a product of elementary factors, or identify a product covering space $\widetilde{\mathcal{Z}}$ such that $\mathcal{M} = \widetilde{\mathcal{Z}}/G$. We prove that factorized encoder-prior distributions over the factors yield product topologies with closed-form, decoupled KL divergences (Propositions~\ref{prop:kl_decoupling}--\ref{prop:product_topology}).
    \item For each factor, select an encoder-prior distribution pair from a catalogue of reparametrizable distributions with tractable KL divergences (Table~\ref{tab:distributions}). We provide pairs for periodic, bounded, and unbounded supports, together with coordinate transformations $\chi$ that allow standard neural networks to output the required non-Euclidean parameters with smooth gradients.
    \item For quotient manifolds $\mathcal{M} = \widetilde{\mathcal{Z}}/G$, construct a $G$-invariant feature map $\rho$ that embeds latent coordinates in Euclidean space for the decoder, ensuring that identified points produce identical decoder outputs. We prove that such maps exist whenever $G$ is finite (Proposition~\ref{prop:quotient}) and give explicit constructions for M\"{o}bius strips, Klein bottles, and projective spaces.
    \item Add anchor constraints that fix the coordinate system relative to the data by pinning known observations to prescribed latent locations, or in repulsive form, create soft topological holes.
\end{enumerate}
Figure~\ref{fig:design_procedure} illustrates this pipeline with a Möbius strip example. The framework preserves the usual VAE benefits: the KL term encourages smooth, well-covered representations within the chosen topology, and the reparametrization trick enables end-to-end gradient-based training. In Section~\ref{sec:experiments}, we validate this pipeline on three synthetic manifolds (cylinder, torus, Möbius strip) with known ground truth, and on two real-image datasets (rotated MNIST and cyclically shifted MNIST). We show that when the prior matches the data topology and the coordinate system is grounded through anchoring, the KL term regularizes the latent representation within the correct manifold rather than against it. Topology-aware models then outperform the Gaussian baseline at all practically relevant regularization strengths in terms of reconstruction error and generative quality.

\begin{figure*}[t]
    \centering
    \includegraphics[width=\textwidth]{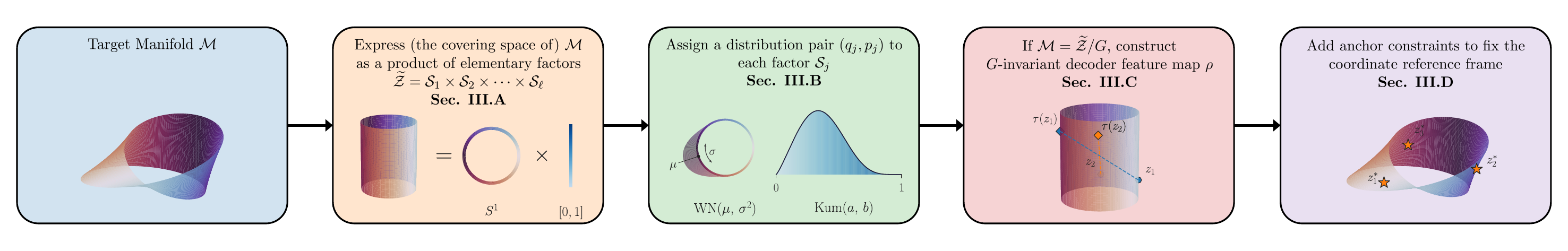}
    \caption{Pipeline for designing a topology-aware VAE, illustrated with the running example of a M\"{o}bius strip. \textbf{Step~1:} Given a target manifold $\mathcal{M}$, find a product covering space $\widetilde{\mathcal{Z}}$ and decompose it into elementary factors. The M\"{o}bius strip is the quotient of a cylinder by a $\mathbb{Z}_2$ action, so $\widetilde{\mathcal{Z}} = S^1 \times [0,1]$. \textbf{Step~2:} Assign a compatible encoder--prior pair from Table~\ref{tab:distributions} to each factor. Here, a Wrapped Normal is assigned to $S^1$ and a Kumaraswamy to $[0,1]$. \textbf{Step~3:} If $\mathcal{M}$ is a quotient, construct a $G$-invariant feature map $\rho$ so that the decoder receives only features that respect the identification; for the M\"{o}bius strip, this ensures $(h,\theta)\sim(1-h,\theta+\pi)$ produce identical outputs. When $\mathcal{M}$ is already a product, $\widetilde{\mathcal{Z}}=\mathcal{M}$ and this step is trivial ($\rho(z)=z$). \textbf{Step~4:} Add anchor constraints to fix the reference frame.}
    \label{fig:design_procedure}
\end{figure*}

The remainder of this paper is organized as follows. Section~\ref{sec:problem} formalizes the topology mismatch problem and reviews the VAE framework. Section~\ref{sec:methods} develops the four tools in the pipeline: product topologies, distribution pairs, quotient topologies, and anchor constraints. Section~\ref{sec:experiments} presents experimental validation on three synthetic datasets with known manifold structure, followed by two real-image experiments. Section~\ref{sec:discussion} discusses conclusions and future directions.

\section{Problem Formulation}
\label{sec:problem}
We begin by formalizing the relationship between data manifolds and latent space topology, then review the VAE framework and identify the mathematical requirements for valid inference with non-standard priors.

\subsection{Data on Low-Dimensional Manifolds}
\label{sec:manifold_assumption}
Many modern data analysis problems involve finding structure in high-dimensional observations that lie on or near a low-dimensional \emph{manifold}~\cite{Lee2011}, a space that locally resembles $\mathbb{R}^n$ but may have a different \emph{topology}, or global shape. For instance, a circle $S^1$ is locally like $\mathbb{R}$ but has a periodic topology that $\mathbb{R}$ does not.

We now formalize the data-generating setting. Let $\mathcal{M} \subseteq \mathbb{R}^n$ be an $n$-dimensional manifold representing the space of underlying factors of variation. High-dimensional observations $y \in \mathbb{R}^p$ (with $n \ll p$) are generated via
\begin{equation}
\label{eq:generative}
y = f(x) + \eta,
\end{equation}
where $x \in \mathcal{M}$ is a point on the manifold, $f: \mathcal{M} \to \mathbb{R}^p$ is the generative function, and $\eta$ represents observation noise. We call $x$ the \emph{true latent variable} associated with observation $y$, and $\mathcal{M}$ the \emph{true latent space}. We assume $f$ is \emph{injective} on $\mathcal{M}$, meaning that distinct true latent variables produce distinguishable observations. When a generative function is not injective, one can sometimes restore injectivity by passing to the quotient space $\mathcal{M}/{\sim}$, where $x_1 \sim x_2$ if and only if $f(x_1) = f(x_2)$.

Given data generated according to~\eqref{eq:generative}, two natural goals arise. The first is \emph{dimensionality reduction}: finding a low-dimensional representation that captures the essential structure of $\mathcal{M}$. The second is \emph{generative modeling}: learning to generate new observations consistent with~\eqref{eq:generative}. VAEs are a tool for achieving both goals simultaneously. However, when $\mathcal{M}$ has a nontrivial topology, standard VAEs face limitations that we address in Section~\ref{sec:topology_mismatch}.

\subsection{Variational Autoencoders}
A VAE~\cite{kingma2014} introduces a \emph{latent space} $\mathcal{Z}$ of dimension $\ell$ and represents each observation $y$ by a \emph{latent variable} $z \in \mathcal{Z}$, learned through an encoder-decoder architecture trained via variational inference. The VAE has three components: an \emph{encoder} $q_\phi(z|y)$, a distribution over $\mathcal{Z}$ parameterized by trainable weights $\phi$; a \emph{decoder} $p_\psi(y|z)$, a distribution over observation space $\mathbb{R}^p$ parameterized by weights $\psi$; and a \emph{prior} $p(z)$ over the latent variables. Figure~\ref{fig:VAE_architecture} illustrates the architecture.

\begin{figure}[t]
    \centering
    \includegraphics[scale=0.55]{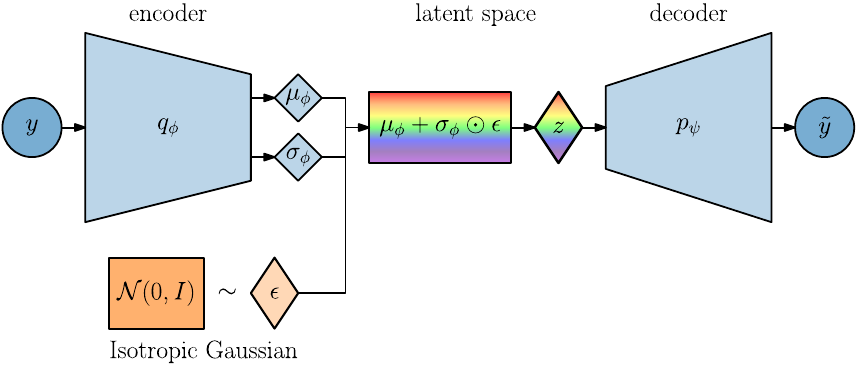}
    \caption{Standard VAE architecture. The encoder networks $\mu_\phi$ and $\sigma_\phi$ produce the parameters of the encoder distribution $q_\phi(z|y)$. A latent sample is drawn via the reparametrization trick: $z = \mu_\phi(y) + \sigma_\phi(y) \odot \epsilon$ with $\epsilon \sim \mathcal{N}(0, I)$. The decoder network $\mu_\psi$ maps $z$ to a reconstruction $\hat{y} = \mu_\psi(z)$.}
    \label{fig:VAE_architecture}
\end{figure}

In the standard VAE, the encoder distribution is a diagonal Gaussian whose mean $\mu_\phi: \mathbb{R}^p \to \mathbb{R}^\ell$ and standard deviation $\sigma_\phi: \mathbb{R}^p \to \mathbb{R}^\ell_{>0}$ are produced by neural networks:
\begin{equation}
q_\phi(z|y) = \mathcal{N}(z;\; \mu_\phi(y),\; \text{diag}(\sigma^2_\phi(y))).
\end{equation}
The decoder distribution is also Gaussian, with mean $\mu_\psi: \mathbb{R}^\ell \to \mathbb{R}^p$ produced by a neural network and a scalar variance $\sigma_\psi^2$ that is either fixed or learned:
\begin{equation}
p_\psi(y|z) = \mathcal{N}(y;\; \mu_\psi(z),\; \sigma_\psi^2 I).
\end{equation}
The decoder mean $\hat{y} = \mu_\psi(z)$ serves as the reconstruction of $y$, and $\sigma_\psi^2$ controls the assumed observation noise level. The standard prior is a zero-mean Gaussian with identity covariance (i.e., isotropic): $p(z) = \mathcal{N}(0, I)$.

The VAE is trained by maximizing the evidence lower bound (ELBO) over a dataset $\{y_i\}_{i=1}^N$. For a single observation $y$, the ELBO is
\begin{equation}
\label{eq:elbo}
\mathcal{L}(\phi, \psi; y) = \underbrace{\mathbb{E}_{q_\phi(z|y)}[\log p_\psi(y|z)]}_{\text{Reconstruction}} - \beta \underbrace{\text{KL}(q_\phi(z|y) \| p(z))}_{\text{Regularization}},
\end{equation}
where $\beta > 0$ is a scalar weight. The Kullback-Leibler (KL) divergence~\cite{cover2006} between two distributions $q$ and $p$ is defined as
\begin{equation}
\label{eq:kl_def}
\text{KL}(q \| p) = \mathbb{E}_{z \sim q}\!\left[\log \frac{q(z)}{p(z)}\right].
\end{equation}
It is nonnegative and equals zero if and only if $q = p$. For the diagonal Gaussian encoder and isotropic Gaussian prior, the KL divergence has the closed form
\begin{equation}
\label{eq:kl_gaussian}
\text{KL}(q_\phi(z|y) \| p(z)) = \frac{1}{2}\sum_{j=1}^\ell \left(\mu_j^2 + \sigma_j^2 - 1 - \log \sigma_j^2\right),
\end{equation}
where $\mu_j = [\mu_\phi(y)]_j$ and $\sigma_j = [\sigma_\phi(y)]_j$ are the $j$-th components of the encoder's mean and standard deviation outputs. This expression decomposes as a sum over latent dimensions because both distributions factorize across coordinates. We will exploit this decomposition property in Section~\ref{sec:product_topologies} to combine different distribution families within a single latent space.

The reconstruction term in~\eqref{eq:elbo} encourages accurate reconstruction of observations, while the KL term regularizes the encoder by encouraging $q_\phi(z|y)$ to remain close to the prior. The weight $\beta$ controls this trade-off: $\beta = 1$ gives the standard VAE, while $\beta \neq 1$ gives the $\beta$-VAE~\cite{Higgins2017}. Without regularization ($\beta = 0$), the optimal encoder collapses to a Dirac delta at the point minimizing reconstruction error for each input, overfitting the training data rather than learning a smooth latent representation.

To enable gradient-based optimization of~\eqref{eq:elbo}, the \textit{reparametrization trick}~\cite{kingma2014,rezende2014} expresses the latent sample as a deterministic function of the encoder parameters and independent noise:
\begin{equation}
z = \mu_\phi(y) + \sigma_\phi(y) \odot \epsilon, \quad \epsilon \sim \mathcal{N}(0, I),
\end{equation}
where $\odot$ denotes element-wise multiplication. Since $\epsilon$ is sampled independently of $\phi$, the gradient $\nabla_\phi \mathbb{E}_{q_\phi}[\log p_\psi(y|z)]$ can be computed by backpropagation through the sampling operation. Section~\ref{sec:methods} generalizes this construction to non-Gaussian distributions.

\subsection{The Topology Mismatch Problem}
\label{sec:topology_mismatch}
Two spaces are \emph{homeomorphic} if there exists a continuous bijection with continuous inverse (a \emph{homeomorphism}) between them, which implies they have the same topology. The standard Gaussian prior $p(z) = \mathcal{N}(0, I)$ has support $\mathcal{Z} = \mathbb{R}^\ell$, which is homeomorphic to Euclidean space. This is appropriate when the true latent space $\mathcal{M}$ is itself homeomorphic to $\mathbb{R}^n$, but creates a fundamental mismatch when $\mathcal{M}$ has a different topology. Many manifolds of practical interest are not homeomorphic to Euclidean space. These include bounded domains such as hypercubes $[0,1]^n$, periodic domains such as tori $T^n = (S^1)^n$, spheres $S^{n-1}$~\cite{davidson2018}, and quotient spaces such as projective spaces and Möbius strips.

To see why the topology of $\mathcal{Z}$ matters, recall that both dimensionality reduction and generative modeling aim to capture the structure of $\mathcal{M}$. For dimensionality reduction, we seek a mapping $g: f(\mathcal{M}) \to \mathcal{Z}$ such that $g \circ f: \mathcal{M} \to \mathcal{Z}$ is a homeomorphism. For such a homeomorphism to exist, $\mathcal{Z}$ must be homeomorphic to $\mathcal{M}$. This is the key requirement that motivates topology-aware VAEs. When $g \circ f$ is a homeomorphism but not the identity, the learned representation captures the manifold structure in a different coordinate system. For generative modeling, we often also desire \emph{interpretability}: specific latent coordinates should correspond to meaningful factors of variation. This requires recovering the original coordinates $g(f(x)) = x$ for all $x \in \mathcal{M}$, which requires additional constraints (Section~\ref{sec:anchoring}).

Before developing the solution, we address a natural question: \emph{is explicit topology shaping even necessary?} Standard VAEs can already learn approximate nontrivial topologies through the reconstruction term alone. For example, a Gaussian VAE trained on data with a periodic latent factor may arrange its latent variables in a ring around the origin, approximating the circular topology. However, this arrangement is not stable under changes in the regularization strength: the KL term in~\eqref{eq:elbo} continuously pushes latent variables toward the prior's center, working against the data's periodic structure. Increasing the KL weight $\beta$ eventually collapses the ring. Setting $\beta = 0$ avoids this collapse but removes all regularization, leading to overfitting and a latent space that does not generalize or support meaningful generation. In contrast, when the prior matches the data topology (e.g., a Wrapped Normal prior for a circular factor), the KL term is compatible with the manifold's structure, allowing strong regularization without topological distortion.

The topology mismatch problem can thus be stated as follows: \emph{given a true latent space $\mathcal{M}$ with known topology, construct a VAE whose latent space $\mathcal{Z}$ is homeomorphic to $\mathcal{M}$, while retaining a tractable ELBO and end-to-end gradient-based training.}

\section{Mathematical Framework for Topology Shaping}
\label{sec:methods}
For the ELBO~\eqref{eq:elbo} to remain tractable when using non-Gaussian distributions, three conditions must be satisfied:
\begin{enumerate}
\renewcommand{\labelenumi}{\Alph{enumi})}
\item The KL divergence $\text{KL}(q_\phi(z|y) \| p(z))$ must be computable, either in closed form or via efficient approximation. This requires $\text{supp}(q_\phi) \subseteq \text{supp}(p)$, since the KL divergence diverges when the encoder assigns positive probability to regions where $p(z) = 0$.
\item Samples $z \sim q_\phi$ must be expressible through the reparametrization trick as $z = T(\phi, \epsilon)$ with $\epsilon$ drawn from a fixed distribution independent of $\phi$, so that the reconstruction term can be optimized via backpropagation.
\item All transformations in the computational graph must be differentiable, so that gradients propagate end-to-end.
\end{enumerate}

We now develop a framework that satisfies these conditions for all manifolds that admit a \emph{product covering space}: manifolds that are either products of elementary factors, or quotients of such products by a finite symmetry group $G$. The four components of the framework (product topologies, distribution pairs, quotient constructions, and anchor constraints) are developed in Sections~\ref{sec:product_topologies}--\ref{sec:anchoring}, with limitations discussed in Section~\ref{sec:limitations}. Figure~\ref{fig:generalized_VAE_architecture} provides an overview of the generalized architecture.

\begin{figure}[t]
    \centering
    \includegraphics[scale=0.55]{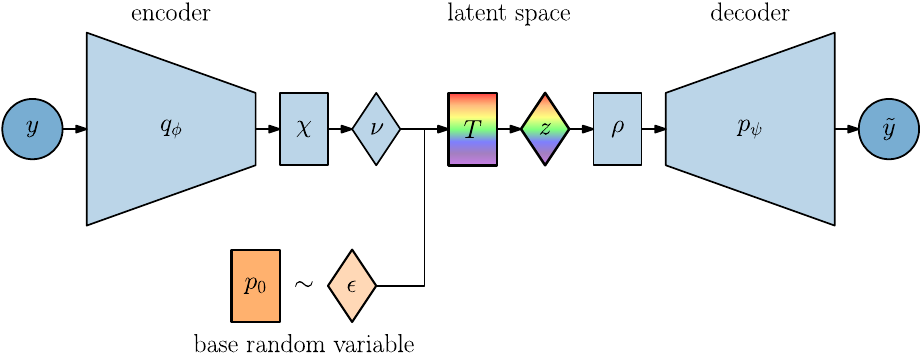}
    \caption{VAE architecture for topology shaping. The encoder maps observations $y$ to distribution parameters $\nu$, which define an encoder distribution $q_\phi(z|y)$ on target support $\mathcal{S}$ via coordinate transformation $\chi$. The reparametrization $T$ transforms base samples $\epsilon \sim p_0$ to latent samples $z$ on the target support $\mathcal{S}$. The decoder receives $z$ through feature map $\rho$, which embeds the latent coordinates in Euclidean space.}
    \label{fig:generalized_VAE_architecture}
\end{figure}

\subsection{Factorized Distributions and Product Topologies}
\label{sec:product_topologies}
The standard Gaussian VAE already exploits a factorized prior and encoder: both $p(z) = \mathcal{N}(0,I)$ and $q_\phi(z|y) = \mathcal{N}(\mu,\text{diag}(\sigma^2))$ decompose across dimensions, and the KL divergence in~\eqref{eq:kl_gaussian} splits into a sum of per-dimension terms. This factorization is not specific to Gaussians. Whenever both encoder and prior factorize across latent coordinates, the KL divergence decomposes into independent per-factor terms, regardless of the distribution families used~\cite{cover2006}. This observation is the foundation of our framework: it allows us to shape the latent topology one coordinate at a time while keeping the KL tractable.
\begin{proposition}[KL Divergence Decoupling]
\label{prop:kl_decoupling}
Let $q(z) = \prod_{j=1}^\ell q_j(z_j)$ and $p(z) = \prod_{j=1}^\ell p_j(z_j)$ be factorized distributions. Then
\begin{equation}
\label{eq:kl_decouple}
\emph{KL}(q \| p) = \sum_{j=1}^\ell \emph{KL}(q_j \| p_j).
\end{equation}
\end{proposition}
See Appendix~\ref{app:kl_decoupling_proof} for the proof.

This decoupling has three important consequences. First, the KL divergence can be computed efficiently as a sum of per-factor divergences, each of which may have a different closed form depending on the distribution family. Second, different distribution families can be mixed across coordinates: one latent dimension can use a Gaussian while another uses a Wrapped Normal. Third, by selecting a distribution pair whose support matches the desired topology for each factor, we can shape the latent space factor by factor. This last point is formalized in the following proposition.

\begin{proposition}[Product Topologies]
\label{prop:product_topology}
Let $q_\phi(z|y) = \prod_{j=1}^\ell q_j(z_j|y)$ and $p(z) = \prod_{j=1}^\ell p_j(z_j)$ be factorized distributions where each pair $(q_j, p_j)$ satisfies $\text{supp}(q_j) \subseteq \text{supp}(p_j)$ and has a computable KL divergence. Then the latent space has the product topology
\begin{equation}
\mathcal{Z} = \mathcal{S}_1 \times \mathcal{S}_2 \times \cdots \times \mathcal{S}_\ell,
\end{equation}
where $\mathcal{S}_j = \text{supp}(p_j)$ is the support of the $j$-th coordinate's prior. The ELBO is tractable with the KL term decomposing as a sum over coordinates through Proposition~\ref{prop:kl_decoupling}.
\end{proposition}

The proof of this proposition is trivial. The support of a product distribution is the Cartesian product of the individual supports, which defines the product topology on $\mathcal{Z}$. Tractability follows from Proposition~\ref{prop:kl_decoupling}.

Although factorized distributions can only produce \emph{product topologies}, the product construction is surprisingly versatile. Provided that reparametrizable distribution pairs with the right supports can be found, Propositions~\ref{prop:kl_decoupling} and~\ref{prop:product_topology} immediately yield tractable VAEs on tori, cylinders, annuli, and arbitrary higher-dimensional products. Davidson et al.~\cite{davidson2019} applied the same principle to hyperspherical latent spaces, replacing a single high-dimensional vMF distribution with a product of lower-dimensional ones to sidestep the concentration bottleneck that arises at high dimension. The construction extends even further: many manifolds that are not themselves products (Möbius strips, Klein bottles, projective spaces) arise as quotients of products by discrete group actions. Constructing suitable distribution pairs for each support type and handling the identifications that quotient topologies require are nontrivial, however, and are the subjects of Sections~\ref{sec:reparametrization} and~\ref{sec:quotient_topologies}.

One alternative is to drop the factorization assumption and estimate the KL divergence via Monte Carlo: $\text{KL}(q \| p) = \mathbb{E}_{z \sim q}[\log q(z) - \log p(z)]$. Normalizing flows~\cite{rezende2015} provide the required reparametrization and tractable log-densities, but Monte Carlo estimation introduces gradient variance, and evaluating log-densities for expressive flows requires computing Jacobian determinants at each sample. We instead maintain factorized distributions with their closed-form KL divergences, and handle non-product topologies through $G$-invariant decoder features in Section~\ref{sec:quotient_topologies}. The next section catalogues the specific distribution pairs available for each elementary factor type.

\subsection{Reparametrizable Distribution Pairs}
\label{sec:reparametrization}
In this section, we catalogue encoder-prior pairs that have computable KL divergences with matching supports, and that admit reparametrizable sampling (satisfying requirements A) and B) from Section~\ref{sec:problem}). The key construction is the reparametrization trick~\cite{kingma2014}: expressing encoder samples as deterministic transformations of samples from a fixed base distribution. Let $\epsilon \sim p_0$ be a base random variable with support $\mathcal{S}_0$, and let $T: \mathcal{S}_0 \times \mathcal{V} \to \mathcal{S}$ be a smooth map parameterized by $\nu \in \mathcal{V}$. Then $z = T(\epsilon; \nu)$ defines a reparametrizable distribution on $\mathcal{S}$ with
\begin{equation}
\label{eq:reparam_general}
\nabla_\nu \mathbb{E}_{z}[h(z)] = \mathbb{E}_{\epsilon \sim p_0}[\nabla_\nu h(T(\epsilon; \nu))],
\end{equation}
for any differentiable function $h$, provided $T$ is differentiable in $\nu$. This moves the stochasticity into $\epsilon$, allowing gradients to flow through $T$.

A classical construction takes $p_0 = \mathcal{U}(0,1)$ and $T = F^{-1}$, where $F$ is the CDF of the target distribution. This yields samples from any continuous distribution with invertible CDF. The Kumaraswamy and Exponential distributions in Table~\ref{tab:distributions} use this approach. Other choices of $p_0$ and $T$ are also valid: the Gaussian uses $p_0 = \mathcal{N}(0,1)$ with affine $T(\epsilon; \mu, \sigma) = \mu + \sigma\epsilon$; the Wrapped Normal applies a wrapping operation to Gaussian samples; the von Mises-Fisher uses rejection sampling with reparametrized proposals~\cite{davidson2018}.

Table~\ref{tab:distributions} lists per-factor distribution pairs used in this paper. Each entry specifies the support, the reparametrization $T$, and whether the KL divergence admits a closed form, an approximation, or requires Monte Carlo estimation. Combined with the factorized KL from Section~\ref{sec:product_topologies}, these pairs allow product topologies to be assembled factor by factor. This framework is not limited to one-dimensional factors: a factor can be a hypersphere $S^{n-1}$ with a von~Mises--Fisher distribution or a Euclidean space $\mathbb{R}^n$ with a multivariate Gaussian, as long as the distributions factorize; Table~\ref{tab:distributions} focuses on low-dimensional building blocks because these suffice for our experiments. Closed-form KL expressions for each pair are provided in Appendix~\ref{app:kl}.

The table also includes the Gumbel-Softmax relaxation~\cite{Jang2017,Maddison2017}, which extends the framework to discrete latent factors via approximate reparametrization of categorical variables, enabling hybrid discrete--continuous latent spaces~\cite{Dupont2018}. All distribution pairs in Table~\ref{tab:distributions} satisfy requirements A) and B). Requirement C) (smooth gradient flow) additionally demands that the encoder's real-valued outputs be mapped to valid parameter spaces without gradient pathologies, which is addressed next.

\begin{table*}[t]
\caption{Distribution Pairs for Topology Shaping}
\label{tab:distributions}
\centering
\begin{tabular}{lllll}
\hline
\textbf{Encoder Distribution} & \textbf{Support} & \textbf{Reparametrization $T$} & \textbf{Prior Distribution} & \textbf{KL} \\
\hline
Gaussian $\mathcal{N}(\mu, \sigma^2)$ & $\mathbb{R}$ & $\mu + \sigma\epsilon$, \quad $\epsilon \sim \mathcal{N}(0,1)$ & $\mathcal{N}(0,1)$ & Closed \\
Exponential $\text{Exp}(\lambda_q)$ & $[0, \infty)$ & $-\log(1-u)/\lambda_q$, \quad $u \sim \mathcal{U}(0,1)$ & $\text{Exp}(\lambda_0)$ & Closed \\
Kumaraswamy $\text{Kum}(a,b)$ & $[0, 1]$ & $(1-(1-u)^{1/b})^{1/a}$, \quad $u \sim \mathcal{U}(0,1)$ & $\mathcal{U}(0,1)$ & Closed \\
Uniform $\mathcal{U}(a,b)$ & $[a, b]$ & $a + (b-a)u$, \quad $u \sim \mathcal{U}(0,1)$ & $\mathcal{U}(a,b)$ & Closed \\
Wrapped Normal $\text{WN}(\mu, \sigma^2)$ & $S^1$ & $\text{wrap}(\mu + \sigma\epsilon)$, \quad $\epsilon \sim \mathcal{N}(0,1)$ & $\mathcal{U}(S^1)$ & Approximate \\
von Mises-Fisher $\text{vMF}(\mu, \kappa)$ & $S^{n-1}$ & Rejection sampling with Householder flow~\protect\cite{davidson2018} & $\mathcal{U}(S^{n-1})$ & MC \\
Gumbel-Softmax $\text{GS}(\pi, t)$ & $\{1,\ldots,K\}$ & $\text{softmax}((\log\pi + g)/t)$, \quad $g \sim \text{Gumbel}(0,1)$ & $\text{Cat}(1/K)$ & Closed \\
\hline
\end{tabular}
\end{table*}

\subsubsection{Encoder Coordinate Transformations}
\label{sec:encoder_coords}
Neural networks produce outputs in $\mathbb{R}^n$, but the distribution parameters from Table~\ref{tab:distributions} may live in non-Euclidean parameter spaces. We therefore need a coordinate transformation $\chi: \mathbb{R}^m \to \mathcal{V}$ that maps the encoder's real-valued outputs to valid parameters on the target support. This $\chi$ must cover all of $\mathcal{V}$ and be smooth with informative gradients everywhere in its domain. Naive approaches fail both requirements: clipping to a bounded range $[a,b]$ creates zero gradients at the boundaries, and taking outputs modulo $2\pi$ for periodic parameters creates gradient discontinuities at the wrap-around point.

For periodic parameters, we use a normalization-based construction. To output an angle $\mu \in S^1$, the encoder produces two real numbers $(c', s') \in \mathbb{R}^2$, which we normalize and convert:
\begin{equation}
(c, s) = \frac{(c', s')}{\|(c', s')\|}, \quad \mu = \arctan2(s, c).
\end{equation}
This defines $\chi: \mathbb{R}^2 \setminus \{0\} \to S^1$, which is smooth with well-defined gradients everywhere except the origin (a set of measure zero that is never reached in practice). For bounded parameters like the standard deviation $\sigma > 0$ of a Wrapped Normal, a softplus transformation $\chi(\cdot) = \log(1 + e^{(\cdot)})$ maps $\mathbb{R} \to \mathbb{R}_{>0}$ smoothly.

At this point, the framework appropriately transforms network outputs through the encoder and the variational sampling operation to produce latent samples on the target support while preserving differentiability. However, the decoder faces a different challenge: its input consists of the sampled latent coordinates themselves, which may have identifications that a standard neural network cannot respect. The next section addresses this.

\subsection{Quotient Topologies via Invariant Features}
\label{sec:quotient_topologies}
After sampling $z$ from the encoder distribution, we pass it to the decoder. In general, the latent coordinates may have \emph{identifications}: distinct coordinate values that represent the same point. If $z_1 \sim z_2$ are identified, the decoder must satisfy $p_\psi(y \mid z_1) = p_\psi(y \mid z_2)$; otherwise, the model treats equivalent points inconsistently. Even the circle $S^1 = \mathbb{R}/\mathbb{Z}$ identifies $\theta \sim \theta + 2\pi$, while the Möbius strip also identifies $(h, \theta) \sim (1-h, \theta+\pi)$. Our solution is to feed the decoder not the raw coordinates $z$, but a transformed representation using a \emph{feature map} $\rho: \widetilde{\mathcal{Z}} \to \mathbb{R}^k$ that is invariant to the identifications.

We formalize identifications using group actions. Let $\widetilde{\mathcal{Z}}$ be a product space with a finite group $G$ acting smoothly on it, and let $\mathcal{Z} = \widetilde{\mathcal{Z}}/G$ be the quotient space where points in the same $G$-orbit are identified: $z_1 \sim z_2$ iff $z_2 = g \cdot z_1$ for some $g \in G$. The space $\widetilde{\mathcal{Z}}$ is called a covering space of $\mathcal{Z}$: informally, it is a ``global unfolding'' of $\mathcal{Z}$ in which identified points are pulled apart. Note that while a given manifold may admit multiple covering spaces, we specifically require any covering space that is itself a product of elementary factors. This way, the factorized distributions from Section~\ref{sec:product_topologies} apply on the encoder side. Each point in $\mathcal{Z}$ then has $|G|$ preimages in $\widetilde{\mathcal{Z}}$, so the encoder operates on the covering space while only the decoder input changes. For $\rho$ to enforce the identifications, it must be (i) $G$-invariant, i.e., $\rho(z) = \rho(g \cdot z)$ for all $g \in G$; (ii) injective on $G$-orbits, so that distinct equivalence classes remain distinguishable; and (iii) smooth, so that gradients can flow through $\rho$ during backpropagation. The following proposition guarantees that such a map always exists when $G$ is finite.

\begin{proposition}[Quotient Topologies via Invariant Features]
\label{prop:quotient}
Let $\widetilde{\mathcal{Z}} \subseteq \mathbb{R}^m$ and let $G$ be a finite group acting smoothly on $\widetilde{\mathcal{Z}}$. Then there exists a smooth map $\rho: \widetilde{\mathcal{Z}} \to \mathbb{R}^k$ that is $G$-invariant and injective on $G$-orbits.
\end{proposition}
\begin{proof}
For any two points $z_1, z_2$ in distinct $G$-orbits, the function $f(z) = \prod_{g \in G} \|z - g \cdot z_1\|^2$ is a polynomial in $z$ (hence smooth), $G$-invariant, vanishes on the orbit of $z_1$, and is strictly positive on the orbit of $z_2$. So $G$-invariant smooth functions separate orbits. Since $\widetilde{\mathcal{Z}} \subseteq \mathbb{R}^m$ has finite covering dimension, finitely many such functions suffice to form an injective map $\rho$ on the orbit space. As each component of $\rho$ is a polynomial, $\rho$ is smooth.
\end{proof}

Proposition~\ref{prop:quotient} guarantees the existence of the required invariant features for any finite group $G$. For constructing $\rho$ in practice, the \emph{Reynolds operator}~\cite{Sturmfels2008} is the central tool. Given any smooth function $f: \widetilde{\mathcal{Z}} \to \mathbb{R}$, its group average
\begin{equation}
R[f](z) = \frac{1}{|G|}\sum_{g \in G} f(g \cdot z)
\end{equation}
is automatically $G$-invariant and smooth. Starting from a basis of smooth functions on the covering space (trigonometric for angular coordinates, polynomial for intervals), one applies $R$ to monomials of increasing degree, discards dependent and vanishing terms, and verifies that the remaining features separate $G$-orbits. Proposition~\ref{prop:quotient} guarantees termination at some finite degree. Appendix~\ref{app:reynolds} gives the full procedure and worked examples. Since evaluating $R$ requires only $|G|$ function evaluations, the construction cost is linear in the group order. Moreover, $\rho$ is determined once at design time; during training it is a fixed closed-form function that adds no computational overhead beyond the increase in decoder input dimension. The following examples illustrate $\rho$ for three quotient spaces.

\begin{example}[Periodic coordinates]
For a circular factor, the standard embedding $\rho(\theta) = (\cos\theta, \sin\theta)$ is the required invariant map. For product spaces without further identifications, $\rho$ applies coordinate-wise: a cylinder $S^1 \times [0,1]$ uses $\rho(\theta, h) = (\cos\theta, \sin\theta, h)$. When the target manifold has additional identifications beyond periodicity, the feature maps compose accordingly.
\end{example}

\begin{example}[M\"{o}bius strip]
The M\"{o}bius strip is $([0,1] \times S^1) / \mathbb{Z}_2$ with identification $(h, \theta) \sim (1-h, \theta + \pi)$. Applying the Reynolds operator to the cylinder decoder features $\{\cos\theta, \sin\theta, h - \tfrac{1}{2}\}$ (all of which flip sign under the deck transformation), one finds that degree-1 invariants vanish and the surviving degree-2 features are (see Appendix~\ref{app:reynolds} for the full derivation):
\begin{equation}
\label{eq:mobius_invariant}
\rho(h, \theta) = \begin{pmatrix} \cos(2\theta) \\ \sin(2\theta) \\ (h - \tfrac{1}{2})^2 \\ \cos(\theta)(h - \tfrac{1}{2}) \\ \sin(\theta)(h - \tfrac{1}{2}) \end{pmatrix}.
\end{equation}
From $\cos 2\theta$ and $\sin 2\theta$, the angle $\theta$ is determined up to $\theta$ or $\theta + \pi$; the cross terms then resolve the remaining ambiguity, so $\rho$ is injective on $\mathbb{Z}_2$-orbits.
\end{example}

\begin{example}[Klein bottle]
The Klein bottle is $T^2 / \mathbb{Z}_2$ with identification $(\theta_1, \theta_2) \sim (\theta_1 + \pi, -\theta_2)$. Applying the Reynolds operator to the torus decoder features $\{\cos\theta_1, \sin\theta_1, \cos\theta_2, \sin\theta_2\}$ yields the degree-2 invariant features:
\begin{equation}
\rho(\theta_1, \theta_2) = \begin{pmatrix} \cos(2\theta_1) \\ \sin(2\theta_1) \\ \cos(\theta_2) \\ \cos(\theta_1)\sin(\theta_2) \\ \sin(\theta_1)\sin(\theta_2) \end{pmatrix}.
\end{equation}
These features satisfy $\rho(\theta_1 + \pi, -\theta_2) = \rho(\theta_1, \theta_2)$, and the cross terms resolve the remaining ambiguity to ensure injectivity on $\mathbb{Z}_2$-orbits.
\end{example}

\begin{remark}
One might also consider the reverse construction: taking quotients factor-by-factor and then forming a product. This yields the same class of manifolds when each group acts on a single factor, since $(\mathcal{S}_1/G_1) \times (\mathcal{S}_2/G_2) \cong (\mathcal{S}_1 \times \mathcal{S}_2)/(G_1 \times G_2)$. The product-then-quotient formulation used here is strictly more general, as it also allows group actions that couple coordinates across factors (e.g., the $\mathbb{Z}_2$ action on the Möbius strip mixes $h$ and $\theta$).
\end{remark}

Sections~\ref{sec:product_topologies}--\ref{sec:quotient_topologies} provide all the tools to build a latent space with a prescribed product or quotient topology. However, the learned coordinate system is typically not unique: any self-homeomorphism of $\mathcal{Z}$ composed with the encoder is equally valid, and nothing in the training objective pins down which homeomorphism is learned.

\subsection{Anchor Constraints}
\label{sec:anchoring}
Anchor constraints resolve this coordinate ambiguity by specifying where particular observations should map in latent space. They can also create approximate topological holes, which is useful when the desired topology cannot be expressed as a product of bounded or periodic supports.

\subsubsection{Reference frame anchoring}
The coordinate ambiguity arises because the target latent space $\mathcal{Z}$ admits self-homeomorphisms: for example, $S^1 \times [0,1]$ is unchanged under any rotation of the angular coordinate, so the encoder and any such rotation of it are equally valid solutions. This issue is noted but left unresolved in prior work on non-Euclidean VAEs~\cite{falorsi2018}. Anchor constraints fix the coordinate frame by penalizing deviation from prescribed locations. Writing $\mu_i = \mathbb{E}_{q_\phi}[z \mid y_i]$ for the posterior mean, the anchor loss is
\begin{equation}
\mathcal{L}_{\text{anchor}} = \frac{\lambda}{M} \sum_{i=1}^{M} \|\mu_i - z_i^*\|^2,
\end{equation}
where $y_1, \ldots, y_M$ are anchor observations and $z_1^*, \ldots, z_M^*$ their prescribed target locations. For many distributions (including Gaussian and Kumaraswamy) the mean $\mu_i$ is available in closed form; otherwise Monte Carlo estimation can be used. The loss generalizes center loss~\cite{wen2016}, which attracts class features toward learnable centers; here the targets are fixed a priori for coordinate alignment.

More generally, anchoring need not cover the full latent vector. When $\mathcal{Z}$ decomposes as a product and only one factor carries the coordinate ambiguity (as with $S^1 \times [0,1]$, where the angular coordinate is rotationally ambiguous but the interval coordinate is not), the norm above need only be computed over the ambiguous coordinates. The remaining dimensions are left free to self-organize through the reconstruction objective. This reduces the number of labeled anchor observations required and avoids over-constraining coordinates with no ambiguity.

To break all continuous symmetries, the anchors must be in general position. If the symmetry group $G$ has orbits of dimension $d$ over the constrained coordinates, at least $d+1$ anchors are needed with no $d$ of them lying in the same orbit. For $S^1 \times [0,1]$ with $G = SO(2)$ acting by rotation of the circle, two anchors at distinct angles suffice.

\subsubsection{Soft topological holes}
Repulsive anchors can create approximate topological holes by penalizing proximity to specified locations. We seek a potential that diverges logarithmically at large distances (providing a long-range repulsive force) but remains finite at the origin (avoiding gradient singularities). A natural choice is the regularized log-potential
\begin{equation}
\mathcal{L}_{\text{repel}} = \frac{\lambda_r}{N} \sum_{i=1}^{N} \log\!\left(\frac{\|z_i - z^*\|^2}{\xi^2} + 1\right),
\end{equation}
where $z_i$ are reparametrized samples from the encoder distribution, $z^*$ is the center of the hole, and $\xi > 0$ controls the effective hole radius. The parameter $\xi$ directly sets the crossover scale: points within distance $\xi$ from $z^*$ experience strong repulsion, while points beyond roughly $2\xi$ experience negligible effect. Unlike reference frame anchoring, which constrains the expected position, the repulsive loss operates on samples to ensure the entire posterior mass avoids the hole.

Unlike bounded distributions (Proposition~\ref{prop:product_topology}), soft constraints cannot guarantee exact holes since the prior still has full support. They are useful when the desired topology cannot be expressed as a product or quotient of products.

\subsection{Achievable Topologies and Limitations}
\label{sec:limitations}
Combining the tools from the preceding sections, the framework covers products (cylinders, tori, hypercubes), quotients of products (Möbius strips, Klein bottles, projective spaces), and single-factor manifolds such as hyperspheres $S^{n-1}$ (via the von Mises-Fisher distribution). Because $SO(3) \cong \mathbb{RP}^3 = S^3/\mathbb{Z}_2$, even rotational latent spaces fall within the achievable class. Discrete factors can also be included via the Gumbel-Softmax relaxation (Section~\ref{sec:reparametrization}). Several limitations remain, however. The topology must be specified a priori; we do not learn it from data. Disconnected manifolds require mixture priors~\cite{Dilokthanakul2017}, breaking KL decoupling and requiring techniques such as importance-weighted bounds~\cite{burda2016}. Finally, manifolds whose universal covering space is not a product of our elementary factors (such as higher-genus surfaces and hyperbolic manifolds) cannot be represented exactly. Nevertheless, many manifolds of practical interest fall within the achievable classes.

The next section empirically validates the framework on several such topologies.

\section{Experiments}
\label{sec:experiments}
We validate the framework in two parts. Section~\ref{sec:exp_synthetic} presents synthetic experiments on cylinders, tori, and Möbius strips, where the true latent space is known, so topological fidelity can be assessed quantitatively. Sections~\ref{sec:exp_rotated_mnist} and~\ref{sec:exp_shifted_mnist} demonstrate the framework on manipulations of MNIST images, specifically rotation and translation, which both induce periodic latent coordinates. Before showing results, we describe the evaluation metrics in Section~\ref{sec:metrics}.

Because different latent topologies interact differently with the KL regularization, a fixed $\beta$ does not yield a fair comparison. We therefore tune $\beta$ separately for each model so that the \emph{unweighted} KL divergence $\text{KL}(q_\phi \| p)$ is approximately equal across all models within each experiment. This ensures that the models operate at comparable regularization strength and that differences in reconstruction or topology metrics are not artifacts of one model being less regularized than another. The resulting $\beta$ values and unweighted KL terms are reported in each experiment's table. All experiments use an 80/20 train/test split, and each topology-aware model is compared to a Gaussian VAE baseline with the same total latent dimensionality $\ell$.

\subsection{Evaluation Metrics}
\label{sec:metrics}
The tables in each experiment report five quantities. The first two, train and test RMS reconstruction error, measure how well the VAE reconstructs observations. The third, the unweighted KL divergence, quantifies regularization strength. These three quantities are standard but do not directly assess whether the learned latent space respects the assumed topology. A model with low reconstruction error might still scramble the latent geometry, and the KL value alone does not reveal whether the prior is a good match for the learned representation. To capture these aspects, we introduce two additional metrics.

The first added metric tests whether the prior covers the latent space meaningfully. If the prior matches the learned latent structure, then a sample from the prior should decode to a plausible observation that the encoder can map back to a similar region. To test this, we sample $z \sim p(z)$, decode to $\hat{y}$, re-encode to $z'$, re-decode to $\hat{y}'$, and compute $\text{RMS}(\hat{y}, \hat{y}')$. We call this the \textit{prior consistency} error. When the prior mismatches the latent topology, samples can land in regions the encoder never maps to, producing high prior consistency error.

The second added metric tests whether pairwise distances between latent variables reflect the true distances on the data manifold. Since our synthetic experiments provide the true latent variables $x_i$, we can compute geodesic distances on $\mathcal{M}$ and compare them to geodesic distances between latent variables. A natural measure for this comparison is Kruskal's stress-1~\cite{Kruskal1964}, which is standard in multidimensional scaling:
\begin{equation}
\sigma_{\text{geo}} = \sqrt{\frac{\sum_{i \neq j} (\bar{d}_{ij} - \bar{\delta}_{ij})^2}{\sum_{i \neq j} \bar{d}_{ij}^2}},
\end{equation}
where $d_{ij}$ is the geodesic distance between true latent variables $x_i, x_j$ on $\mathcal{M}$ and $\delta_{ij}$ is the geodesic distance between latent variables $z_i, z_j$ on the model's latent space. Bars denote mean-normalization: $\bar{d} = d / \mathbb{E}[d]$, which removes scale differences. We call this the \textit{geodesic stress}. For the cylinder, the latent geodesic uses the product metric on $S^1 \times [0,1]$; for the torus, the product of two angular distances. For quotient manifolds, the geodesic distance is the minimum over the group orbit: $d_{\mathcal{M}}(z_1, z_2) = \min_{g \in G} d_{\widetilde{\mathcal{Z}}}(z_1, g \cdot z_2)$. For the Möbius strip with $G = \mathbb{Z}_2$, this amounts to computing two covering-space distances and taking the smaller one. Lower stress indicates better preservation of the manifold's distance structure.

\subsection{Synthetic Validation: Cylinder, Torus, and M\"{o}bius Strip}
\label{sec:exp_synthetic}
\begin{figure*}[t]
    \centering
    \begin{minipage}{\textwidth}
        \centering
        \includegraphics[scale=0.33]{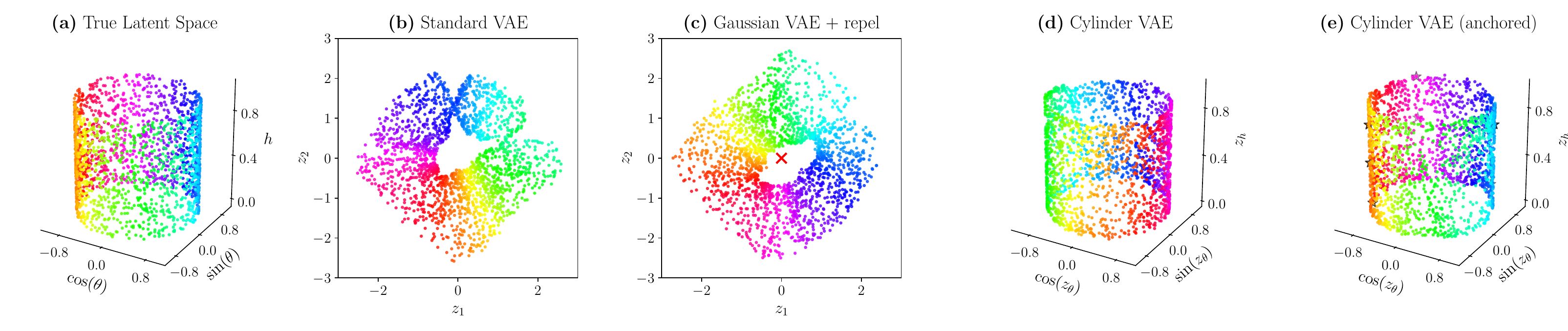}
        \par\vspace{0.4ex}
        \textbf{(a) Cylinder} $S^1 \times [0,1]$: true factors, Gaussian VAE, Cylinder VAE, Cylinder VAE (anchored). Anchor points marked~$\star$.
    \end{minipage}
    \vspace{1.2ex}

    \begin{minipage}{\textwidth}
        \centering
        \includegraphics[scale=0.33]{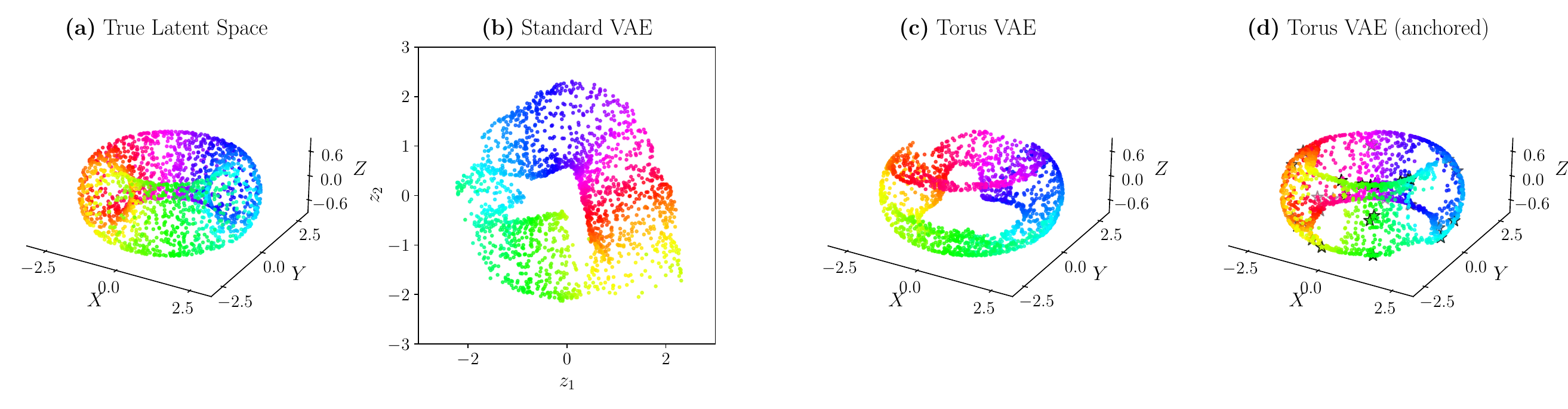}
        \par\vspace{0.4ex}
        \textbf{(b) Torus} $T^2 = S^1 \times S^1$: true factors on the torus surface, Gaussian VAE, Torus VAE, Torus VAE (anchored). Anchor points marked~$\star$.
    \end{minipage}
    \vspace{1.2ex}

    \begin{minipage}{\textwidth}
        \centering
        \includegraphics[scale=0.33]{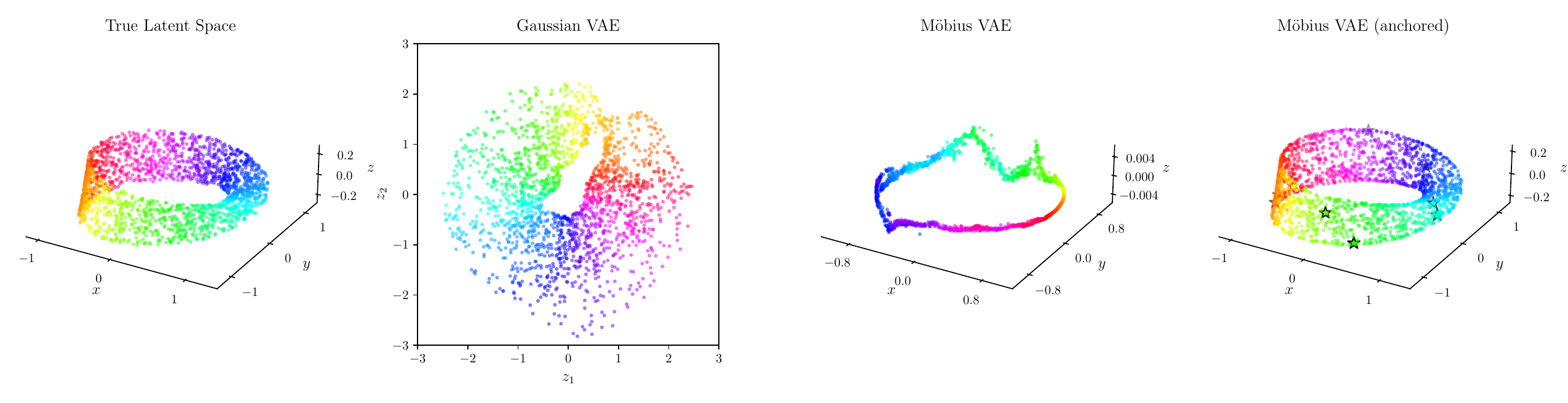}
        \par\vspace{0.4ex}
        \textbf{(c) M\"{o}bius strip}: true factors, Gaussian VAE, M\"{o}bius VAE (unanchored), M\"{o}bius VAE (anchored). Points are canonicalized to the fundamental domain and embedded on the strip surface; rainbow coloring by angular coordinate shows the smoothness of the learned representation.
    \end{minipage}
    \caption{Learned latent representations for the three synthetic experiments (columns: true factors, Gaussian baseline, topology-aware VAE, topology-aware VAE with anchoring). Three-dimensional panels show embeddings in ambient $\mathbb{R}^3$ space. For the cylinder~(a), the embedding is given by $(\cos\theta, \sin\theta, h)$ or $(\cos z_\theta, \sin z_\theta, z_h)$. For the torus~(b), the embeddings use standard parametrizations with ambient coordinate axes. The M\"{o}bius strip in Row~(c) first maps each point to its canonical representative in the fundamental domain $[0,1] \times [-\pi, \pi)$, then applies the standard embedding $(\cos\theta, \sin\theta, h)$ to place it on the strip surface.}
    \label{fig:synthetic}
\end{figure*}
For each of the three target manifolds, we draw $N = 2000$ true latent variables $x \in \mathcal{M}$ uniformly and map each to an observation $y \in \mathbb{R}^{50}$ through a fixed, nonlinear and injective generative network $f_\star$ (three layers with tanh and LeakyReLU nonlinearities, fixed weights) with additive Gaussian noise of standard deviation 0.1, as in~\eqref{eq:generative}. The network is shared across all three experiments, only the input coordinates differ. Models are trained with $\beta$ tuned so that the unweighted KL is approximately equal within each experiment. Table~\ref{tab:metrics_synthetic} collects reconstruction and topology metrics across all three experiments. Figure~\ref{fig:synthetic} shows the corresponding learned latent representations.

\begin{table*}[t]
\caption{Synthetic experiments: reconstruction and topology metrics. Lower is better for all quantities. Models within each group are evaluated at matched KL divergence.}
\label{tab:metrics_synthetic}
\centering
\begin{tabular}{llccccc}
\hline
\textbf{Topology} & \textbf{Model} & \textbf{Train RMS} & \textbf{Test RMS} & \textbf{KL} & \textbf{Consistency RMS} & \textbf{Geodesic Stress} \\
\hline
\multirow{4}{*}{Cylinder $S^1\!\times\![0,1]$}
 & Gaussian            & 0.141 & 0.150 & 4.36 & 0.113 & 0.296 \\
 & Gaussian + repel    & 0.140 & 0.143 & 4.35 & 0.100 & 0.333 \\
 & Cylinder            & 0.140 & 0.148 & 4.37 & 0.085 & 0.082 \\
 & Cylinder (anch.)    & 0.137 & 0.144 & 4.40 & 0.086 & 0.065 \\
\hline
\multirow{3}{*}{Torus $T^2$}
 & Gaussian            & 0.165 & 0.188 & 5.03 & 0.131 & 0.394 \\
 & Torus               & 0.159 & 0.173 & 5.17 & 0.175 & 0.250 \\
 & Torus (anch.)       & 0.156 & 0.173 & 5.12 & 0.126 & 0.231 \\
\hline
\multirow{3}{*}{M\"{o}bius strip}
 & Gaussian            & 0.117 & 0.122 & 4.37 & 0.112 & 0.327 \\
 & M\"{o}bius              & 0.167 & 0.209 & 3.93 & 0.078 & 0.294 \\
 & M\"{o}bius (anch.)      & 0.116 & 0.128 & 4.48 & 0.100 & 0.067 \\
\hline
\end{tabular}
\end{table*}

\subsubsection{Cylinder}
\label{sec:exp_cylinder}
The cylinder $S^1 \times [0,1]$ is the simplest mixed-product example: one periodic coordinate and one bounded coordinate. We draw $\theta \sim \mathcal{U}(-\pi, \pi)$ and $h \sim \mathcal{U}(0, 1)$, embed the pair in $\mathbb{R}^3$ as $(\cos\theta, \sin\theta, h)$, and apply $f_\star$. Four models are compared: a standard Gaussian VAE; a Gaussian VAE augmented with a repelling anchor at the origin (soft topological hole); an unanchored Cylinder VAE with a Wrapped Normal encoder for $\theta$ and a Kumaraswamy encoder for $h$; and an anchored version with eight corner anchors at $\theta \in \{-\pi, -\frac{\pi}{2}, 0, \frac{\pi}{2}\}$ and $h \in \{0, 1\}$.

Figure~\ref{fig:synthetic}(a) shows the learned representations. The Gaussian VAE cannot represent the periodic and bounded structure of the cylinder: the angular coordinate wraps without constraint and the height is unbounded in the latent space, so points that coincide at $\theta = \pm\pi$ are assigned different latent positions. The Cylinder VAE correctly captures both structural features of the target: the height coordinate is confined to $[0,1]$ by the Kumaraswamy support and the angular coordinate is periodic. Without anchoring, however, the coordinate system is arbitrary and the learned angular frame may be rotated relative to the true one; adding anchor constraints aligns the learned frame with the true latent variables.

At matched KL ($\approx 4.4$), consistency error drops by 25\% (0.085 vs.\ 0.113) and geodesic stress drops by 72\% (0.082 vs.\ 0.296). The two metrics tell complementary stories: the lower consistency error means the cylindrical prior covers the learned latent space rather than placing probability mass in unused regions, while the lower geodesic stress means the intrinsic distance structure of the data manifold is better preserved. The repelling-anchor Gaussian baseline slightly improves consistency (0.100 vs.\ 0.113) but its geodesic stress (0.333) is worse than the standard Gaussian, confirming that a soft hole at the origin cannot substitute for the correct topology. Anchoring the Cylinder VAE further reduces geodesic stress to 0.065.

\subsubsection{Torus}
\label{sec:exp_torus}
The torus $T^2 = S^1 \times S^1$ is a purely periodic product: both latent coordinates are angles. We draw $\theta_1, \theta_2 \sim \mathcal{U}(-\pi, \pi)$, embed the pair on a torus surface in $\mathbb{R}^3$ via $((R + r\cos\theta_2)\cos\theta_1,\ (R + r\cos\theta_2)\sin\theta_1,\ r\sin\theta_2)$ with major radius $R = 2$ and minor radius $r = 0.8$, and apply $f_\star$. Three models are compared: a Gaussian VAE; a Torus VAE with Wrapped Normal encoders for both angular coordinates; and an anchored Torus VAE with 16 anchor points on a $4 \times 4$ grid over $[-\pi, \pi]^2$.

Figure~\ref{fig:synthetic}(b) shows the learned representations. The Gaussian VAE produces boundary artifacts at $\theta_1, \theta_2 \approx \pm\pi$: points that are close on the torus surface but near opposite ends of the angular range are placed far apart in the Gaussian latent space, because the Gaussian prior does not wrap either dimension. The Torus VAE distributes points smoothly across the full $[-\pi, \pi]^2$ range.

At matched KL ($\approx 5.1$), geodesic stress falls by 37\% (0.250 vs.\ 0.394). Note that the unanchored Torus VAE scores worse in terms of consistency RMS (0.175) than the Gaussian baseline (0.131), despite having a topologically matched prior. The explanation is that the uniform prior on $T^2$ assigns equal probability to all angular configurations, so the encoder is free to rotate the coordinate frame to any orientation. Without a fixed reference, the decoder must generalize across all such rotations simultaneously, which is harder than fitting a single fixed frame. Anchoring resolves this ambiguity and yields the best consistency (0.126) together with the lowest geodesic stress (0.231).

\subsubsection{M\"{o}bius Strip}
\label{sec:exp_mobius}
The M\"{o}bius strip is homeomorphic to $([0,1] \times S^1) / \mathbb{Z}_2$ under the identification $(h, \theta) \sim (1-h, \theta + \pi)$. We generate observations by sampling $\theta \sim \mathcal{U}(-\pi, \pi)$ and $h \sim \mathcal{U}(0, 1)$ on the covering cylinder, computing the $\mathbb{Z}_2$-invariant features $\rho(h, \theta)$ from~\eqref{eq:mobius_invariant}, and applying $f_\star$ to those features. Because $\rho(h, \theta) = \rho(1-h, \theta + \pi)$, the two covering-space preimages of each M\"{o}bius point produce the same observation up to noise, which is the correct generative semantics for the quotient identification.

The M\"{o}bius VAE encodes into the covering cylinder $S^1 \times [0,1]$ and passes the $\mathbb{Z}_2$-invariant features to the decoder. Because each M\"{o}bius point has two preimages, the KL divergence requires a correction: writing $r = q_C(\tau(h,\theta)) / q_C(h,\theta)$ for the density ratio at identified points (where $\tau(h, \theta) = (1-h, \theta+\pi)$ is the deck transformation), the adjusted M\"{o}bius KL is
\begin{equation}
\text{KL}(q_M \| p_M) = \text{KL}(q_C \| p_C) - \log 2 + \mathbb{E}[\log(1 + r)].
\end{equation}
See Appendix~\ref{app:mobius_kl} for the derivation.

Figure~\ref{fig:synthetic}(c) shows the learned representations. The Gaussian VAE cannot capture the non-orientable structure of the M\"{o}bius strip at all. The M\"{o}bius VAE without anchoring correctly uses the M\"{o}bius topology in principle, but across many random seeds it frequently collapses the height coordinate toward $h \approx 0.5$, effectively discarding it to reduce the KL term. This degeneracy produces train RMS 0.167, well above the Gaussian baseline (0.117). The collapsed representation stays internally self-consistent (consistency RMS 0.078) while reconstructing poorly, because it has discarded the height coordinate. The anchored M\"{o}bius VAE does not exhibit this collapse and achieves train RMS 0.116, on par with the Gaussian while preserving the full M\"{o}bius structure. An interesting aspect about the anchored model is that the anchor penalty is a non-negative addition to the objective. For this reason, the global optimum cannot be lowered by adding it. The RMS improvement therefore implies that the anchors steer optimization away from the degenerate local solution that traps the unanchored model. With the collapse resolved, the anchored model also reduces geodesic stress sharply relative to the Gaussian (0.067 vs.\ 0.327) and improves prior consistency (0.100 vs.\ 0.112).

\subsubsection{KL Sweep Analysis}
\label{sec:kl_sweep}
\begin{figure*}[t]
    \centering
    \includegraphics[scale=0.43]{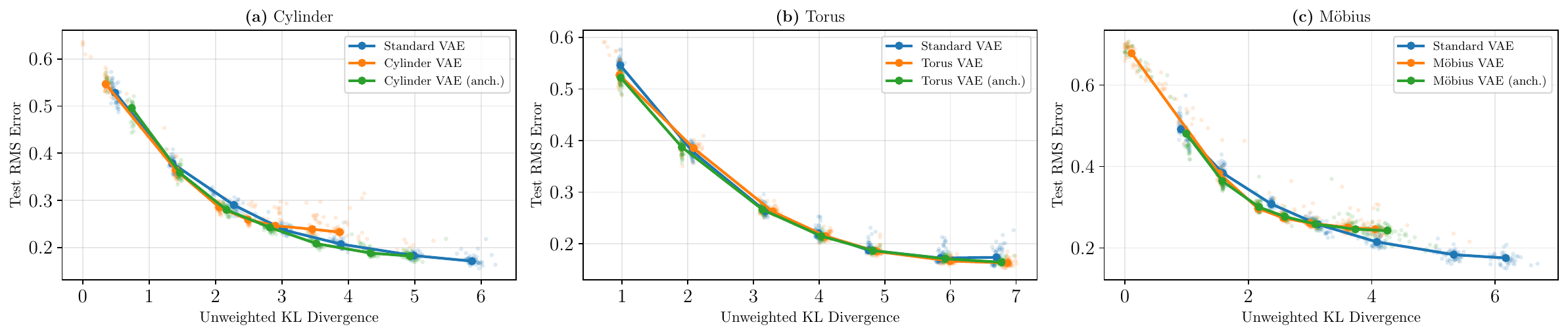}
    \caption{KL sweep for the three synthetic topologies. Each subplot varies $\beta$ from 0.01 to 10.0 and plots reconstruction RMS vs.\ KL divergence for the topology-aware VAE and the Gaussian baseline: (a)~Cylinder $S^1\times[0,1]$, (b)~Torus $T^2$, (c)~M\"{o}bius strip. Shaded bands show the 25th--75th percentile over 50 seeds. The torus topology-aware VAE Pareto-dominates the Gaussian baseline; for cylinder and M\"{o}bius strip, the advantage holds once $\beta$ exceeds a small topology-dependent threshold.}
    \label{fig:kl_sweep}
\end{figure*}
To assess whether the topology advantage persists across the full reconstruction/regularization tradeoff, we vary the KL weight $\beta$ from 0.01 to 10.0 and train both the topology-aware VAE and the Gaussian baseline at each setting, averaging over 50 random seeds per topology~\cite{Burgess2018}. Figure~\ref{fig:kl_sweep} plots reconstruction RMS against unweighted KL divergence for each topology.

For the torus, the topology-aware VAE Pareto-dominates the Gaussian baseline across essentially all $\beta$ values. For the cylinder and M\"{o}bius strip, the topology-aware VAE outperforms the Gaussian once $\beta$ exceeds a small threshold (below 0.1--0.3 for both). Only at very low $\beta$, where the KL term is negligible relative to reconstruction, does the Gaussian baseline match or slightly outperform the topology-aware model. For the M\"{o}bius strip the anchored variant is the relevant comparison: the unanchored model suffers from the height collapse and does not consistently outperform the Gaussian at moderate $\beta$. These results confirm that the benefit of topology shaping is structural rather than specific to a particular $\beta$ setting, and that it is present whenever the KL term plays a meaningful role in training.

The synthetic experiments validated the framework on low-dimensional observations generated from known manifolds. We now turn to real images to assess whether the topology benefits transfer to a higher-dimensional, more complicated setting.

\subsection{Rotated MNIST ($\mathbb{R}^2 \times S^1$)}
\label{sec:exp_rotated_mnist}
In this experiment, each image is a $28 \times 28$ MNIST digit rotated by a random angle $\theta \sim \mathcal{U}(-\pi, \pi)$. The two generative factors are the digit class and the rotation angle, a single $S^1$ factor. We use digit classes 6 and 7, which are both asymmetric under rotation so that the $S^1$ coordinate is identifiable, with 500 exemplars per class, each rotated 180 times, giving $N = 180{,}000$ images with $y_i \in \mathbb{R}^{784}$.

All three models have the same total latent dimensionality $\ell = 3$. For the custom VAE, the learned latent variable is $z = (z_1, z_2, z_\theta)$, where $z_\theta \in S^1$ is the circular coordinate and $(z_1, z_2) \in \mathbb{R}^2$ are the Euclidean style coordinates; for the Gaussian, all three coordinates are Euclidean. The three models are:
\begin{enumerate}
    \item Gaussian VAE with $\ell = 3$, which yields a latent topology of $\mathbb{R}^3$
    \item Unanchored Mixed-Circle VAE, which yields a latent topology of $\mathbb{R}^2 \times S^1$;
    \item Anchored Mixed-Circle VAE, which yields a latent topology of $\mathbb{R}^2 \times S^1$.
\end{enumerate}
The Mixed-Circle VAE uses a Wrapped Normal encoder for $z_\theta$ and a standard Gaussian encoder for $(z_1, z_2)$. The anchors are applied only on the circular coordinate, see Section~\ref{sec:anchoring}, with 30 evenly spaced target angles. The coordinates $(z_1, z_2)$ are unconstrained and organize by digit class through reconstruction pressure alone.

Figure~\ref{fig:rotated_mnist} compares the three models in four rows. The first two rows confirm that the rotation angle and digit class are encoded in the latent space. The Circle VAE assigns the digit rotation to a native $S^1$ coordinate $z_\theta$, while the Gaussian distributes it across all three Euclidean dimensions. We see that the anchored version retains the orientation of the true data. Observe from row 2 that all models separate the two digit classes into clearly separated clusters. Row 3 shows a 12-step decoded ring per digit class: the style coordinates $(z_1, z_2)$ are fixed at the class centroid in latent space and $z_\theta$ is swept uniformly over $S^1$. Fixing the centroid isolates the rotation factor and restricts the displayed digits to a single class. The Circle VAE produces coherent digits at every angle and wraps back cleanly, while the Gaussian does not. Row 4 shows decoded geodesics: two endpoints from the same digit class are linearly interpolated in latent space and decoded at each step. The Circle VAE produces in-distribution reconstructions throughout, while the Gaussian does not.

Table~\ref{tab:metrics_rotated_mnist} reports reconstruction and topology metrics. At matched KL (approximately 15--17 nats), the anchored Mixed-Circle VAE achieves lower reconstruction error than the Gaussian (train RMS 0.115 versus 0.121), with prior-consistency error falling from 0.117 to 0.043 and geodesic stress falling from 0.625 to 0.325, showcasing the advantage of the correct topology in a real-data setting.

\begin{table}[t]
\caption{Rotated MNIST Experiment: Reconstruction and Topology Metrics. Lower is better for all metrics.}
\label{tab:metrics_rotated_mnist}
\centering
\begin{tabular}{lccccc}
\hline
\textbf{Model} & \textbf{Train} & \textbf{Val} & \textbf{KL} & \textbf{Consistency} & \textbf{Geodesic} \\
 & \textbf{RMS} & \textbf{RMS} & & \textbf{RMS} & \textbf{Stress} \\
\hline
\shortstack[l]{Gaussian\\($\ell=3$)} & 0.121 & 0.126 & 16.54 & 0.117 & 0.625 \\
\shortstack[l]{Mixed-Circle\\VAE} & 0.119 & 0.125 & 15.46 & 0.076 & 0.501 \\
\shortstack[l]{Mixed-Circle\\VAE (anch.)} & 0.115 & 0.122 & 15.69 & 0.043 & 0.325 \\
\hline
\end{tabular}
\end{table}

\begin{figure*}[t]
    \centering
    \includegraphics[width=\textwidth]{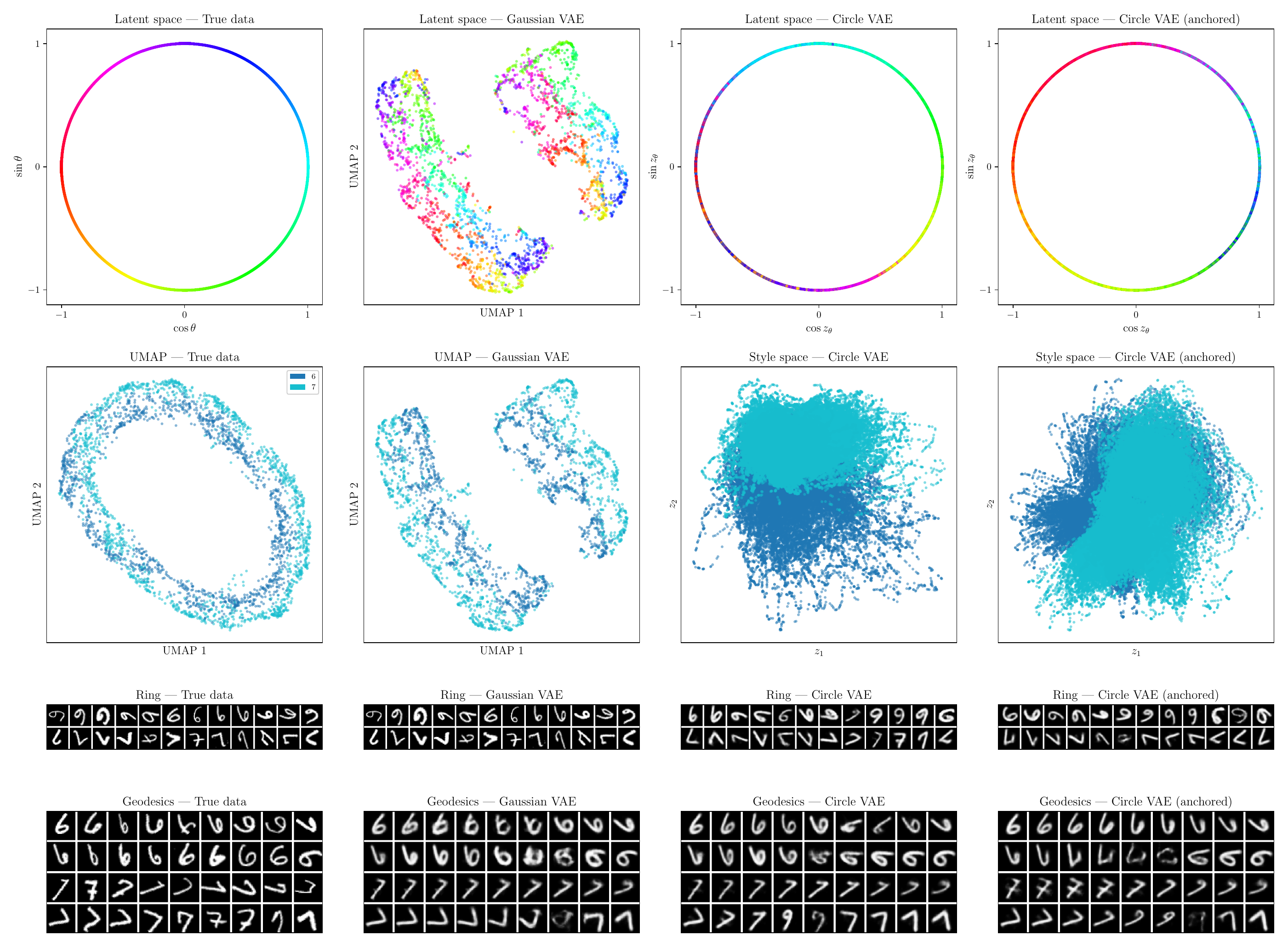}
    \caption{Rotated MNIST experiment ($\mathbb{R}^2 \times S^1$ latent space, digit classes 6 and 7). Columns: ground truth, Gaussian VAE ($\ell=3$), Mixed-Circle VAE, and Mixed-Circle VAE (anchored). First row: latent coordinates colored by the true rotation angle $\theta$. for the Mixed-Circle VAEs, the circular coordinates $(\cos z_\theta, \sin z_\theta)$ are plotted; for the Gaussian, a UMAP projection of the full latent space is shown instead; for the ground-truth column, $(\cos\theta, \sin\theta)$ is plotted. Second row: latent coordinates colored by digit class. The Mixed-Circle VAE shows the Euclidean coordinates $(z_1, z_2)$, while a UMAP projection is shown for the ground-truth and Gaussian columns. Third row: 12-step ring for each digit class (3 and 4), sweeping $z_\theta$ for the Circle VAEs or the proxy $\mathrm{atan2}(z_1, z_2)$ for the Gaussian, with style coordinates $(z_1, z_2)$ fixed at the class centroid. Fourth row: decoded geodesics between two endpoints from the same digit class, linearly interpolated in latent space.}
    \label{fig:rotated_mnist}
\end{figure*}

\subsection{Shifted MNIST ($\mathbb{R}^2 \times T^2$)}
\label{sec:exp_shifted_mnist}
In this experiment, images are generated by cyclically shifting MNIST digits from classes 3, 4, and 7 by random integer pixel offsets $(d_x, d_y)$ with $d_x, d_y \in \{0, \ldots, 27\}$ using \texttt{numpy.roll}. The true latent coordinates are $\theta_i = (2\pi d_i/28 + \pi) \bmod 2\pi - \pi \in [-\pi, \pi)$ for $i \in \{1, 2\}$, placing the pair $(\theta_1, \theta_2)$ exactly on $T^2 = S^1 \times S^1$. We use 500 exemplars per digit class with 100 random shifts each, giving $N = 150{,}000$ images.

All three models have the same total latent dimensionality $\ell = 4$. The Mixed-Torus VAE has learned latent variable $z = (z_1, z_2, z_{\theta_1}, z_{\theta_2})$, with Wrapped Normal encoders on the circle coordinates $z_{\theta_1}, z_{\theta_2} \in S^1$ and a standard Gaussian encoder on the style coordinates $(z_1, z_2) \in \mathbb{R}^2$. The Gaussian VAE uses a flat $\mathbb{R}^4$ prior.

As in the rotated MNIST experiment (Section~\ref{sec:exp_rotated_mnist}), only the angular coordinates $z_{\theta_1}$ and $z_{\theta_2}$ are anchored, while the style coordinates $(z_1, z_2)$ are unconstrained. Anchor targets correspond to pixel shifts $(d_x, d_y) \in \{0, 7, 14, 21\}^2$, giving 16 exact angular targets per digit class.

Figure~\ref{fig:shifted_mnist} compares the three models in four rows. The first two rows confirm that both shift angles and the digit class are encoded well by all VAEs. The Mixed-Torus VAEs assign both shift angles to native $S^1$ coordinates, while the Gaussian distributes them across all four Euclidean dimensions. The anchored VAE provides a significant benefit in terms of interpretability over the unanchored version, as the anchored model's learned coordinates align exactly to the true parameter space. Row 2 shows that all models separate the three digit classes. Row 3 shows a $10 \times 10$ decoded grid, using the same class-centroid approach as in Section~\ref{sec:exp_rotated_mnist}. The anchored Mixed-Torus VAE almost exactly reproduces the ground-truth shift pattern, while the Gaussian shows no periodic structure. Row 4 shows decoded geodesics for each digit class, constructed as in Section~\ref{sec:exp_rotated_mnist}; the Mixed-Torus VAE produces smooth interpolations throughout, while the Gaussian does not.

Table~\ref{tab:metrics_shifted_mnist} reports reconstruction and topology metrics. At matched KL (approximately 19--21 nats), reconstruction quality is similar across all models (train RMS 0.153--0.159). The anchored Mixed-Torus VAE attains the lowest prior-consistency error (0.086, compared to 0.135 for the Gaussian and 0.114 for the unanchored model) and the lowest geodesic stress (0.164, compared to 0.362 for the unanchored model and 0.404 for the Gaussian).

\begin{table}[t]
\caption{Shifted MNIST Experiment: Reconstruction and Topology Metrics. Lower is better for all metrics.}
\label{tab:metrics_shifted_mnist}
\centering
\begin{tabular}{lccccc}
\hline
\textbf{Model} & \textbf{Train} & \textbf{Val} & \textbf{KL} & \textbf{Consistency} & \textbf{Geodesic} \\
 & \textbf{RMS} & \textbf{RMS} & & \textbf{RMS} & \textbf{Stress} \\
\hline
\shortstack[l]{Gaussian\\($\ell=4$)} & 0.159 & 0.173 & 19.70 & 0.135 & 0.404 \\
\shortstack[l]{Mixed-Torus\\VAE} & 0.155 & 0.171 & 20.69 & 0.114 & 0.362 \\
\shortstack[l]{Mixed-Torus\\VAE (anch.)} & 0.153 & 0.174 & 19.58 & 0.086 & 0.164 \\
\hline
\end{tabular}
\end{table}

\begin{figure*}[t!]
    \centering
    \includegraphics[width=\textwidth]{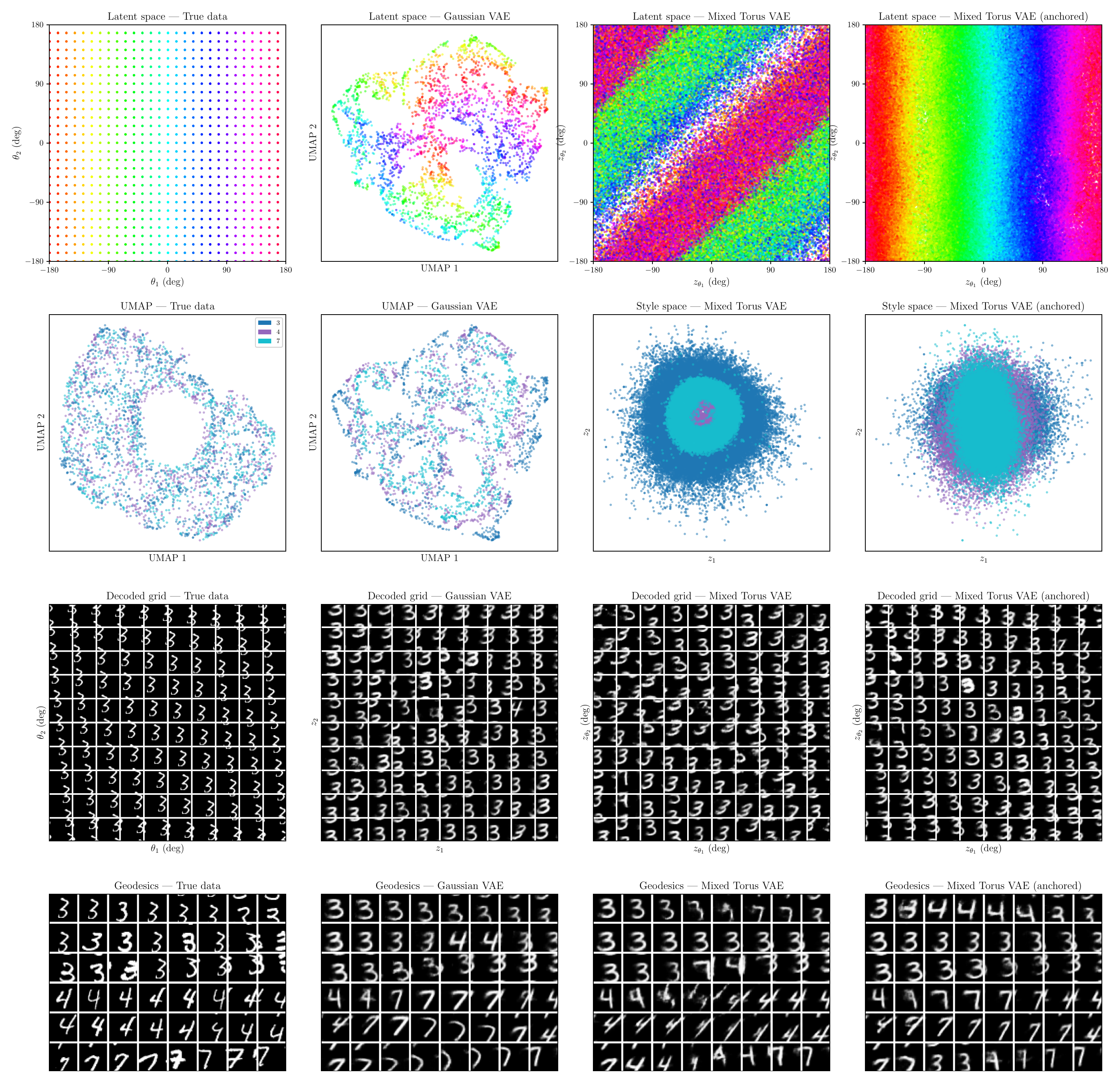}
    \caption{Shifted MNIST experiment ($\mathbb{R}^2 \times T^2$ latent space, digit classes 3, 4, and 7). Columns: ground truth, Gaussian VAE ($\ell=4$), Mixed-Torus VAE, and Mixed-Torus VAE (anchored). First row: latent coordinates colored by the true shift angle $\theta_1$. The ground-truth column plots $(\theta_1, \theta_2)$ directly. The Mixed-Torus VAEs plot the learned circular coordinates $(z_{\theta_1}, z_{\theta_2})$. The Gaussian column shows a UMAP projection of the full latent space. Second row: latent coordinates colored by digit class. The Mixed-Torus VAEs show the Euclidean style coordinates $(z_1, z_2)$. The ground-truth and Gaussian columns show a UMAP projection. Third row: $10 \times 10$ decoded grid, shown for digit 3. The Mixed-Torus VAEs sweep both circular coordinates $(z_{\theta_1}, z_{\theta_2})$ with style coordinates fixed at the digit-3 centroid. The Gaussian VAE sweeps its first two Euclidean dimensions $(z_1, z_2)$, with the remaining dimensions fixed at the digit-3 centroid. Fourth row: decoded geodesic paths, one per digit class, with style coordinates fixed at the class centroid.}
    \label{fig:shifted_mnist}
\end{figure*}

\section{Discussion and Conclusions}
\label{sec:discussion}
We have presented a framework for designing VAE latent spaces with prescribed non-Euclidean topologies. The design procedure has four steps: decompose a target manifold into elementary factors, assign a compatible encoder-prior distribution pair to each, construct a feature map $\rho$ that embeds latent coordinates in Euclidean space for the decoder, and add anchor constraints when coordinate alignment is needed. Factorized distributions yield product topologies with closed-form KL divergences, and $G$-invariant feature maps handle quotient identifications (Proposition~\ref{prop:quotient}). All transformations are differentiable, preserving gradient flow throughout.

Across all five experiments, the anchored topology-aware model improves both geodesic stress and prior consistency relative to the Gaussian baseline. Anchoring the coordinates has three benefits. It aligns the learned coordinates with the true generative factors for a more interpretable latent space. It resolves coordinate ambiguities that arise even under a matched prior, such as the rotational symmetry of the torus. It also steers optimization away from degenerate minima, as in the M\"{o}bius strip, where the unanchored model collapses the height coordinate. A KL sweep across 50 seeds confirms that this advantage is not specific to a single $\beta$. The topology-aware models achieve lower reconstruction error at every active regularization level, and the gap closes only when $\beta$ is negligible and the prior has no influence on the representation.

Two lessons can be taken from the experiments. First, a prior whose support matches the data manifold is necessary but not sufficient for a faithful representation. The coordinate frame within that support must also be pinned down by anchoring, as the torus and Möbius results make clear. Second, the improvements in the prior-consistency metric imply increased generative quality. Samples from the prior decode to observations that re-encode to the same region, so a low value means prior samples that stay consistent with the data. The MNIST experiments demonstrate this, as the topology-aware models decode plausible digits across the whole manifold while the Gaussian fails at the periodic boundary.

Several directions for future work remain. The topology must currently be specified a priori. Learning it from data instead is an open problem. The invertible transformation perspective connects naturally to normalizing flows~\cite{Brehmer2020}, where the transformation is learned rather than fixed. The quotient topology approach could be extended to continuous symmetry groups or to learning the group action from data. Combining topology shaping with disentanglement objectives may enable latent spaces that are both topologically faithful and interpretable.

\bibliographystyle{IEEEtran}
\bibliography{references}

\appendices
\section{KL Divergence Derivations}
\label{app:kl}
This appendix provides the proof of Proposition~\ref{prop:kl_decoupling}, closed-form KL divergences for the distribution pairs in Table~\ref{tab:distributions}, and the derivation of the Möbius strip KL divergence, which requires special treatment due to the quotient topology.

\subsection{KL Divergence Decoupling (Proof of Proposition~\ref{prop:kl_decoupling})}
\label{app:kl_decoupling_proof}
By definition of KL divergence and the factorization assumption:
\begin{align}
\text{KL}(q \| p) &= \int \prod_{j=1}^\ell q_j(z_j) \log \frac{\prod_{j=1}^\ell q_j(z_j)}{\prod_{j=1}^\ell p_j(z_j)} \, dz \nonumber \\
&= \int \prod_{j=1}^\ell q_j(z_j) \sum_{k=1}^\ell \log \frac{q_k(z_k)}{p_k(z_k)} \, dz \nonumber \\
&= \sum_{k=1}^\ell \int q_k(z_k) \log \frac{q_k(z_k)}{p_k(z_k)} \prod_{j \neq k} q_j(z_j) \, dz \nonumber \\
&= \sum_{k=1}^\ell \text{KL}(q_k \| p_k),
\end{align}
where the last step uses $\int q_j(z_j) \, dz_j = 1$ for all $j \neq k$.

\subsection{Gaussian-Gaussian KL}
The KL divergence between univariate Gaussians $q = \mathcal{N}(\mu_q, \sigma_q^2)$ and $p = \mathcal{N}(\mu_p, \sigma_p^2)$ is:
\begin{equation}
\text{KL}(q \| p) = \log\frac{\sigma_p}{\sigma_q} + \frac{\sigma_q^2 + (\mu_q - \mu_p)^2}{2\sigma_p^2} - \frac{1}{2}.
\end{equation}
For the standard prior $p = \mathcal{N}(0, 1)$ used in vanilla VAEs, this simplifies to $\frac{1}{2}(\mu_q^2 + \sigma_q^2 - 1 - \log \sigma_q^2)$.

\subsection{Exponential-Exponential KL}
For the half-infinite latent space $\mathbb{R}_{\geq 0}$, we use the Exponential distribution. Given $q = \text{Exp}(\lambda_q)$ and $p = \text{Exp}(\lambda_p)$ with rate parameters $\lambda_q, \lambda_p > 0$:
\begin{equation}
\text{KL}(q \| p) = \log\frac{\lambda_q}{\lambda_p} + \frac{\lambda_p}{\lambda_q} - 1.
\end{equation}
When $q = p$ (i.e., $\lambda_q = \lambda_p$), the KL divergence is zero as expected.

\subsection{Kumaraswamy-Uniform KL}
The Kumaraswamy distribution provides a tractable encoder for latent variables on bounded domains. For $q = \text{Kum}(\alpha, \beta)$ and the uniform prior $p = \mathcal{U}(0, 1)$, we obtain:
\begin{equation}
\text{KL}(q \| p) = \left(1 - \frac{1}{\beta}\right) + \left(1 - \frac{1}{\alpha}\right)\psi(\beta) + \log(\alpha\beta)
\end{equation}
where $\psi$ is the digamma function.

\subsection{Wrapped Normal-Uniform KL}
The Wrapped Normal distribution $q = \text{WN}(\mu, \sigma^2)$ on $S^1$ is constructed by wrapping a Gaussian onto the circle:
\begin{equation}
z = \arctan2(\sin(\mu + \sigma\epsilon), \cos(\mu + \sigma\epsilon)),
\end{equation}
where $\epsilon \sim \mathcal{N}(0,1)$. This construction inherits exact reparametrization from the Gaussian.

The KL divergence between $q = \text{WN}(\mu, \sigma^2)$ and the circular uniform prior $p = \mathcal{U}(S^1)$ admits a closed-form approximation. Since the uniform prior has constant density $p(z) = 1/(2\pi)$, the KL divergence simplifies to
\begin{align}
\text{KL}(q \| p) &= \int q(z) \log q(z) \, dz - \int q(z) \log p(z) \, dz \\
&= -H(q) + \log(2\pi),
\end{align}
where $H(q)$ denotes the differential entropy of the Wrapped Normal. For small $\sigma$, the wrapping has negligible effect and $q$ is approximately Gaussian. Using the Gaussian entropy $H(\mathcal{N}(\mu, \sigma^2)) = \frac{1}{2}\log(2\pi\sigma^2) + \frac{1}{2}$, we obtain
\begin{align}
\text{KL}(q \| p) &\approx \log(2\pi) - \frac{1}{2}\log(2\pi\sigma^2) - \frac{1}{2}\\
&= \frac{1}{2}\left(\log(2\pi) - \log(\sigma^2) - 1\right).
\end{align}
For large $\sigma$, the wrapped distribution approaches uniform on the circle, so $H(q) \to \log(2\pi)$ and $\text{KL}(q \| p) \to 0$. We combine these regimes by clamping:
\begin{equation}
\text{KL}(q \| p) \approx \max\left(0, \frac{1}{2}\left(\log(2\pi) - \log(\sigma^2) - 1\right)\right).
\end{equation}
This approximation is accurate when the wrapped normal is well-approximated by either a Gaussian ($\sigma \lesssim 1$) or uniform ($\sigma \gtrsim 2$). In the intermediate regime $\sigma \in [1, 2]$, the true KL divergence transitions smoothly to zero, hence the clamped formula provides a reasonable approximation throughout.

\subsection{Möbius Strip KL Divergence}
\label{app:mobius_kl}
The Möbius strip $M = (S^1 \times [0,1])/\mathbb{Z}_2$ is the quotient of the cylinder $C = S^1 \times [0,1]$ by the action $\tau(h, \theta) = (1-h, \theta+\pi)$. This action identifies opposite edges with a twist, creating the Möbius topology. The encoder distribution $q_C$ lives on the cylinder $C$, while the uniform prior $p_M$ lives on $M$. We derive the KL divergence $\text{KL}(q_M \| p_M)$, where $q_M$ is the pushforward of $q_C$ under the quotient map.

Let $p_C$ denote the uniform distribution on the cylinder. Since $M$ has half the area of $C$, we have $p_M([z]) = 2 p_C(\{z, \tau(z)\})$ for each equivalence class $[z] \in M$. The encoder density at $[z] \in M$ is $q_M([z]) = q_C(z) + q_C(\tau(z))$, summing contributions from both identified points.

The KL divergence on $M$ is given by
\begin{align}
\text{KL}(q_M \| p_M) &= \int_M q_M([z]) \log \frac{q_M([z])}{p_M([z])} \, d[z] \nonumber \\
&= \int_C q_C(z) \log \frac{q_C(z) + q_C(\tau(z))}{2 p_C(z)} \, dz \nonumber \\
&= \text{KL}(q_C \| p_C) - \log 2 + \mathbb{E}_{z \sim q_C}[\log(1 + r(z))],
\end{align}
where $r(z) := q_C(\tau(z))/q_C(z)$ is the density ratio at identified points. The correction term $-\log 2 + \mathbb{E}[\log(1 + r)]$ accounts for the quotient structure. When the encoder concentrates mass at one of the identified points only (occupying half the cylinder), we obtain $r \to 0$ and the correction approaches $-\log 2$, reflecting that $M$ has half the volume of $C$. When the encoder places equal mass on both identified points ($r = 1$), the correction vanishes. In practice, we compute the density ratio $r(z)$ in log-space for numerical stability, using the log-probabilities from the Wrapped Normal and Kumaraswamy distributions. The expectation $\mathbb{E}[\log(1 + r)]$ is approximated via Monte Carlo sampling from $q_C$.

\section{Constructing Feature Map $\rho$ Using the Reynolds Operator}
\label{app:reynolds}
Section~\ref{sec:quotient_topologies} introduced the Reynolds operator $R$ for constructing $G$-invariant feature maps $\rho$. Here we give the step-by-step construction procedure and work through the Möbius strip derivation in full.

To construct an orbit-separating feature map $\rho$ for a quotient $\mathcal{Z} = \widetilde{\mathcal{Z}}/G$, proceed as follows:
\begin{enumerate}
    \item Choose a basis of coordinate functions for the covering space. For each factor of $\widetilde{\mathcal{Z}}$, select functions $f_1, \ldots, f_m$ that naturally parameterize it:
    \begin{itemize}
        \item \emph{Circular factor $S^1$:} use $\{\cos\theta, \sin\theta\}$.
        \item \emph{Interval factor $[0,1]$:} use $\{h - \tfrac{1}{2}\}$.
        \item \emph{Unbounded factor $\mathbb{R}$ or $[0,\infty)$:} use $\{z\}$ (or $\{z - c\}$ centered at a symmetry-compatible point $c$).
    \end{itemize}
    The full basis for the covering space is the union of these per-factor functions. Good choices make the group action act simply (e.g., as sign flips) on as many basis elements as possible, which streamlines the subsequent steps.
    \item Form monomials in $f_1, \ldots, f_m$ up to some degree $d$.
    \item Apply $R$ to each monomial.
    \item Discard any result that vanishes identically, and remove linearly dependent features.
    \item Verify that the surviving features separate $G$-orbits (i.e., $\rho(z_1) = \rho(z_2)$ implies $z_1 \sim z_2$). If not, increase $d$ and repeat.
\end{enumerate}
Proposition~\ref{prop:quotient} guarantees termination at some finite degree $d$. All components of $\rho$ are polynomials or trigonometric polynomials by construction, so they are smooth and their gradients are available in closed form for backpropagation.

\subsection{Application to the Möbius Strip}
The Möbius strip is $([0,1] \times S^1) / \mathbb{Z}_2$ with deck transformation $\tau(h,\theta) = (1-h, \theta+\pi)$. Following Step~1 of the procedure, the covering space is a cylinder $[0,1] \times S^1$, so we select $f_1 = \cos\theta$, $f_2 = \sin\theta$ for the circular factor and $f_3 = h - \tfrac{1}{2}$ for the interval factor (centering at $\tfrac{1}{2}$ ensures that $\tau$ acts as a sign flip on $f_3$). Under $\tau$, all three basis elements flip sign: $\cos\theta \mapsto -\cos\theta$, $\sin\theta \mapsto -\sin\theta$, and $h - \tfrac{1}{2} \mapsto -(h-\tfrac{1}{2})$.

Since each basis element is odd under $\tau$, the Reynolds operator applied to any degree-1 monomial vanishes: $R[\cos\theta] = \tfrac{1}{2}(\cos\theta + (-\cos\theta)) = 0$, and likewise for $\sin\theta$ and $h - \tfrac{1}{2}$. At degree 2, products of two odd functions are even, so they survive. The non-redundant degree-2 invariants are:
\begin{align*}
R[\cos^2\theta] &= \cos^2\theta = \tfrac{1}{2}(1 + \cos 2\theta), \\
R[\sin^2\theta] &= \sin^2\theta = \tfrac{1}{2}(1 - \cos 2\theta), \\
R[\cos\theta \cdot \sin\theta] &= \cos\theta\sin\theta = \tfrac{1}{2}\sin 2\theta, \\
R[(h-\tfrac{1}{2})^2] &= (h - \tfrac{1}{2})^2, \\
R[\cos\theta \cdot (h-\tfrac{1}{2})] &= \cos\theta \cdot (h - \tfrac{1}{2}), \\
R[\sin\theta \cdot (h-\tfrac{1}{2})] &= \sin\theta \cdot (h - \tfrac{1}{2}).
\end{align*}
The first three are linearly dependent with $\{\cos 2\theta, \sin 2\theta, 1\}$ (and the constant 1 carries no information), leaving the five independent features~\eqref{eq:mobius_invariant}.

To verify orbit separation: from $\cos 2\theta$ and $\sin 2\theta$, the angle $\theta$ is determined up to $\theta$ or $\theta + \pi$. If $\theta_2 = \theta_1$, the cross terms force $h_2 = h_1$ (identical points). If $\theta_2 = \theta_1 + \pi$, the cross terms force $h_2 = 1 - h_1$, which is exactly the $\mathbb{Z}_2$ identification. So $\rho$ is injective on orbits.

\end{document}